\documentclass{article}
\usepackage[utf8]{inputenc}
\usepackage{graphicx}
\usepackage{amsmath}
\usepackage{amsfonts}
\usepackage{amsthm}
\usepackage{amsmath,amssymb} 
\usepackage{color}

\def\given{{\hskip1pt|\hskip1pt}}
\def\argmax{\mathop{\rm argmax}}
\def\argmin{\mathop{\rm argmin}}

\def\qedsymbol{\hfill$\blacksquare$\\[2mm]}

\newcounter{ExampleCount}
\newenvironment{example}
	{\vspace{5pt}\refstepcounter{ExampleCount}\textbf{Example \arabic{ExampleCount}.}}
	{\qedsymbol\vspace{10pt}}

\newenvironment{example*}
	{\vspace{5pt}\refstepcounter{ExampleCount}\textbf{Example \arabic{ExampleCount}.}}
	{\qedsymbol\vspace{10pt}}

\newtheorem{theorem}{Theorem}
\newtheorem{definition}{Definition}

\begin{document}

\title{How to Formulate and Solve Statistical Recognition and Learning Problems}
\author{M.I.Schlesinger,  schles@irtc.org.ua \\
        International Research and Training Center of \\Informational Technologies and Systems\\ 
       Ukrainian Academy of Sciences\\
       Kiev
       \and
       E.V.Vodolazskiy, waterlaz@gmail.com \\
       International Research and Training Center of \\Informational Technologies and Systems\\ 
       Ukrainian Academy of Sciences\\
       Kiev}

\maketitle

\begin{abstract}%
We formulate problems of statistical recognition and learning in a common framework of complex hypothesis testing. 
Based on arguments from multi-criteria optimization, we identify strategies that are improper for solving these problems
and derive a common form of the remaining strategies. We show that some widely used approaches to recognition and learning are improper in this sense.
We then propose a generalized formulation of the recognition and learning problem which embraces the whole range of sizes of the learning sample, including the zero size. 
Learning becomes a special case of recognition without learning. 
We define the concept of closest to optimal strategy, being a solution to the formulated problem, and describe a technique for finding such a strategy. 
On several illustrative cases, the strategy is shown to be superior to the widely used learning methods based on maximal likelihood estimation.

\end{abstract}

\section{Introduction }
\label{SecIntroduction}

We consider three wide directions in theoretical and applied statistics: complex hypothesis testing \cite{chernoff, kendal, lehman, morris}, 
the empirical Bayesian approach \cite{casella, neyman2Breaks, robbinsAssymptotical, robbinsEmpirical} and
learning in pattern recognition \cite{duda, webb}. 
Each of these fields is well-known and the list of relevant publications would be longer than this paper. 
However, the methods
developed in the three fields are not fully consistent with one another.
Indeed, the most popular learning methods seem to have been developed independently of complex hypothesis testing and sometimes even contradict it. 
It seems as if the well-known ideas of self-learning or unsupervised learning \cite{dempsterEM, schlez10lec, schlezEM} appeared independently of the empirical Bayesian approach. 
Such an isolated existence of learning methods results in their rather fragile foundation.

Learning in pattern recognition deals with the situation when the statistical model of an object to be recognized is not defined uniquely and only a set of models is known that includes the true model. 
The common feature of modern learning methods is that a so-called learning sample is given and based on this sample a certain model is selected from the given set. 
Then the recognition strategy is derived as if the selected model was the true one.
Usually, the selected model is a consistent estimate of the true model and therefore this approach is acceptable, provided that the learning sample is large enough. 
However, if the learning sample has limited size the approach gives no guarantee for subsequent recognition.
The situation is known as the "small sample problem". 

Thus there is an obvious gap in our knowledge of statistical recognition under uncertain statistical model. 
Strictly speaking, satisfactory clarity is achieved only in two extreme cases. 
When no learning sample is available the methods for complex hypothesis testing are applicable, such as the minimax approach. 
When a large enough learning sample is available the learning methods are applicable. 
Practical situations fall in the gap between these two extreme cases because any sample has a certain size, it is not arbitrarily large.
The situation is discussed in more detail  in \cite[page 272]{schlez10lec}.

The paper fills this gap. Using the concepts of multi-criteria optimization \cite{ehrgott2005, geoffrion1968}, we analyze the empirical Bayesian approach,  
learning in pattern recognition and complex hypothesis testing in a unified framework.
Based on this unification, we formulate the learning problem as a special case of recognition without learning, 
the latter in turn being a special case of complex hypothesis testing. The proposed formulation covers the whole range of learning sample sizes, including zero size. 
Thus, the  small sample problem disappears as an independent problem and is embraced by the proposed formulation.

The paper is organized as follows:

{\bf Section \ref{SecComplexHypothesis}.} Recognition without learning is considered under an uncertain statistical model of an object to be recognized. 
The concepts of improper and Bayesian strategies are defined and the dichotomy theorem is proved that each strategy is either Bayesian or improper. 
It is shown that the strategies based on maximum likelihood model estimate are improper. The concept of closest to optimal recognition is defined.

{\bf Section \ref{EmpiricalBayesian}.} The seminal ideas of H.Robbins \cite{robbinsAssymptotical}, 
which addressed recognition of compound objects and initiated the empirical Bayesian approach \cite{neyman2Breaks}, are revisited. 
Robbins' ideas are formalized as finding a closest  to optimal strategy. 
It is shown that strategies based on the well-known EM-algorithms differ from the closest to optimal ones and are improper for some objects.

{\bf Section \ref{PatternRecognitionLearning}.} Recognition with learning is formulated as a special case of recognition without learning. 
It is shown that the widely used strategies based on maximal likelihood model estimate are improper. 
It is also shown that the minimax approach to statistical decision making is unsuitable for  solving the learning problem.  
The concept of closest to optimal learning is defined and its formal properties are analyzed.

{\bf Section \ref{SubBayesianLearning}.} The technique of closest to optimal recognition is described.

{\bf Section \ref{Testing}.} Closest to optimal learning procedures are tested on several illustrative examples. The tests include supervised and unsupervised learning. 
The test with supervised learning shows that the closest to optimal learning procedure is superior to the procedure based on maximum-likelihood model estimation. 
The tests with unsupervised learning compare the closest to optimal procedures with EM-algorithms and some heuristic methods. The proposed approach yields the best result.

{\bf Section \ref{Novelty}.} The results of the paper as well as their consequences are summarized.

The results of Sections \ref{SecComplexHypothesis}, \ref{EmpiricalBayesian} and \ref{PatternRecognitionLearning} are mostly negative, 
arguing that the commonly used approaches have fundamental drawbacks and thus new approaches are needed. 
The positive results concern mainly the concept of closest to optimal learning and are described in Sections \ref{PatternRecognitionLearning}, \ref{SubBayesianLearning} and  \ref{Testing}.

\section{Complex object recognition}
\label{SecComplexHypothesis}

The main concepts of complex hypothesis discrimination have been formulated in the general statistical decision theory \cite{chernoff, kendal, lehman, morris}. 
Our aim is to relate the result of this theory to statistical recognition and machine learning \cite{duda, webb}. 
Therefore we formulate the main concepts of statistical decision theory using pattern recognition terminology. 
Basic concepts of our consideration are a set of observable signals, a set of hidden states and a set of models. All three sets are assumed to be finite. 
This simplification allows to articulate the main ideas of the article using the simplest mathematical tools.
Allowing some of the sets to be infinite might obscure the main ideas by more fine mathematical formulations.  
Nevertheless, some examples below involve intervals of real numbers. Our theoretical results are applied to such examples without further ado. 
Probability density is simply used instead of probability.

Let an object to be recognized be characterized by two parameters $x$ and $y$ that take values from finite sets $X$ and $Y$ respectively. 
The parameter $x$ is an observable signal generated by the object. 
The parameter $y$ is a hidden state of the object. 
The signal $x$ and the state $y$ form a random pair, so that a probability distribution $p:X \times Y \rightarrow \mathbb{R}$
over the set of pairs $(x, y)$ exists. However, the probability distribution $p$ may be either known or not known and this makes the distinction between simple and complex objects.

\begin{definition}
A simple object is represented by a triple $\left\langle X, Y, p:X\times Y\rightarrow\mathbb{R}\right\rangle$
where $X$ is a set of signals, $Y$ is a set of states, $p(x,y)$ is a joint probability of signal $x \in X$ and state $y \in Y$.
\end{definition}

We consider more complex cases when the probability distribution $p$ is not known and only a class
of distributions is known that contains $p$.
This uncertainty is represented by a finite set $\Theta$ of models $\theta$, so that
the probability of a pair $(x,y)$ is defined for each model $\theta$ by $p(x,y;\theta)$.

\begin{definition} \label{DefComplex}
A complex object is represented by a quadruple $\left\langle X, Y, \Theta, p:X\times Y \times \Theta \rightarrow\mathbb{R}\right\rangle$
where $X$ is a set of signals, $Y$ is a set of states, $\Theta$ is a set of models and
$p(x,y;\theta)$ is a joint probability of signal $x \in X$ and state $y \in Y$ for the model $\theta \in \Theta$.
\end{definition}

The above-defined complex object is not to be confused with what can be called a {\em pseudo-complex object\/}.
The latter takes place when the model $\theta$ is a random variable with a known a priori probability distribution $p_\Theta{:}\ \Theta \rightarrow \mathbb{R}$. 
In fact, a pseudo-complex object does not differ from a simple one because for such an object the joint probability
\begin{equation} \nonumber
p(x,y)=\sum_{\theta \in \Theta} p_\Theta(\theta) \, p(x,y;\theta)
\end{equation}
is defined for every pair  $x \in X$, $y \in Y$. 
In a complex object the model $\theta$ is not random but fixed although unknown. 
It is only known that the model belongs to a given set.
This paper deals with complex objects, neither with simple ones nor with pseudo-complex ones.

Recognition means making a reasonable decision about the hidden state based on the observed signal. 
Below, we will also consider cases of the decision about current object state being based not only on the currently observed signal 
but on the base of additional information called learning information. 
In order to discriminate these two cases we use a term recognition without learning when the decision is based only on the observed signal and 
recognition with learning when the decision is based both on the signal and on the learning information.
  
Recognition without learning is formalized by a function $q_X{:}\ Y \times X\rightarrow \mathbb{R}$, called a randomized strategy. 
The values  $q_X(y'\given x),y'\in Y, x\in X$, of the function $q_X$ are conditional probabilities. 
Let $Q_X$~be the set of all possible strategies, that is, the set of functions $q_X{:}\ Y \times X\rightarrow \mathbb{R}$ satisfying $q_X(y'\given x)\geq 0$ for every pair $y'\in Y$, $x\in X$ and
$\sum\limits_{y' \in Y}q_X(y'\given x)=1$ for every $x\in X $. 
The strategy $q_X{:}\ Y \times X\rightarrow \mathbb{R}$ defines recognition as follows: if the signal $x$ is observed then the decision 
"the object is in state $y'$" is made with probability $q_X(y'\given x)$. 
The reasonableness of a strategy is formalized by a function $w{:}\ Y \times Y \rightarrow \mathbb{R}$, called a loss function. 
Its value $w(y,y')$, $y \in Y$, $y' \in Y$, is the loss of making a decision that the object state is $y'$ when the true object state is $y$.

The statistical theory of complex hypothesis testing is an appropriate mathematical formalization for recognition of complex objects, not only simple ones. 
However, strategies that are widely used in recognition practice have almost nothing in common with this formalization. 

\begin{example}\label{letterexample}
Let~$x$ be a picture and $y \in Y $ be the name of a letter drawn on the picture.
The letters are drawn by some known person. A priori probabilities $p_Y(y), y \in Y$, of the letters are known for this person,
as well as the conditional probabilities $p_{X\given Y}(x \given y)$.
This is the case of a simple object.

Suppose now that there is a known finite team $\Theta$ of persons $\theta$ drawing pictures and 
the person $\theta$ that provides a picture for recognition is picked randomly with known probability $p_\Theta(\theta)$.
It is necessary to recognize the letter drawn on the picture, not the person who drew it.
One can assume that a priori probabilities $p_Y(y)$, $y \in Y $, of the letters do not depend on the person. 
However, the picture $x$ depends both on the letter $y$ and on the person $\theta \in \Theta$ who drew it.
This dependency is given by conditional probabilities $p_{X\given Y}(x\given y;\theta)$, defined for every picture $x \in X$, 
every letter  $y \in Y $ and every person $\theta \in \Theta$.
This is the case of a pseudo-complex object. 

In fact, this case does not differ from the case of a simple object. 
Indeed, one may substitute the whole team $\Theta$ of persons with a single abstract person $\theta^*$ who draws a random picture $x$ of a letter $y$ with conditional probability 
$p^*(x|y)=\sum\limits_{ \theta \in \Theta } p_\Theta(\theta)p_{X\given Y}(x\given y;\theta)$.

Finally, let us consider the case when all pictures presented for recognition are drawn by a single person $\theta^0$ from a known team $\Theta$. However, this person is unknown. 
So, the model $\theta^0$ is fixed, not random, and a concept of a priori probability distribution $p_{\Theta}$ becomes unappropriated in this case.
Same as before, it is necessary to recognize the letter, not the person who drew it.
This is the case of a complex object.

Commonly used algorithms for this task make the decisions 
	\begin{equation} \nonumber
		y^*=\argmax _{y \in Y} \max_{\theta \in \Theta} p(x\given y;\theta) \quad\text{or}\quad
		y^*=\argmax _{y \in Y} \max_{\theta \in \Theta} p(x,y;\theta).
	\end{equation}
	Neither of these two decisions follow from the statistical decision theory. 
	Moreover, we will show later that both contradict this theory. 
\end{example}

Statistical recognition problems for simple and complex objects are essentially different from each other in nature. 
In the case of a simple object, every strategy $q_X$ is characterized by a single number, namely the expected value $R_X(q_X)$ of the losses,
\begin{equation} \nonumber
R_X(q_X)=\sum_{x \in X} p(x) \sum_{y' \in {Y}} q_X(y'\given x) \sum_{y \in {Y}} p(y\given x) \, w(y,y'),
\end{equation}
called the risk of strategy $q_X$. This leads to the known problem of Bayes risk (single-criterion) minimization over the strategy set. 
In the case of a complex object, every strategy $q_X$ is characterized by $|\Theta|$ numbers $R_X(q_X,\theta)$, $\theta \in \Theta$,  namely the risks of the strategy with respect to every model,
\begin{equation} \label{RiskDef}
R_X(q_X,\theta)=\sum_{x \in X} \sum_{y' \in {Y}} \sum_{y \in {Y}} q_X(y'\given x) \, p(y,x;\theta) \, w(y,y').
\end{equation}
The complex object case goes beyond the scope of single-criterion optimization and is subject to multi-criteria optimization. 
We follow its main ideas (see, for example, the book~\cite{ehrgott2005}) and translate them into the context of our problem. 
First of all, we define some particularly bad strategies that do not optimize any reasonable criterion and thus they are useless. 

\begin{definition} \label{def:improper}
A strategy $q_X'$ predominates a strategy $q_X''$ if the strong inequality $R_X(q_X',\theta) < R_X(q_X'',\theta)$ holds for every model $\theta \in \Theta$; 
a strategy $q_X^0$ is called improper if a strategy $q_X'$ exists that predominates $q_X^0$.
\end{definition}
One can see that the definition of improper strategy is rather strong. It does not embrace all strategies that can be treated as bad ones. For example, each strategy
$\argmin\limits_{q_X \in Q_X} R_X(q_X,\theta^*)$, $\theta^* \in \Theta$, 
is not improper but its usefulness is questionable because each such strategy takes into account only one single model $\theta^*$ and ignores all the others. 
Therefore, if some strategy is not improper it does not mean that it is unconditionally suitable for concrete application. 
However, if a strategy is proved to be improper it shows that it has a serious defect independently of the expected application. 

We want to exclude all improper strategies from consideration and derive the common form of all other strategies. 
Such a dichotomy is possible due to the known result in multi-criteria optimization \cite[Theorem 3.5]{ehrgott2005}. 
The theorem should be central to complex object recognition both with and without learning. 
However, the  known recognition methods have been developed and continue being developed as if the  theorem did not exist. 
This is not surprising because the theorem has been formulated and proved in its own conceptual and terminological context, which is far from pattern recognition. 
We translate and even prove the theorem using the concepts of our article.  

Let $\tau {:}\  \Theta \rightarrow \mathbb{R}$ be a function, called the weight function, satisfying $\tau (\theta) \geq 0$ for every $\theta \in \Theta$ and $\sum\limits_{\theta \in \Theta} \tau (\theta)=1$. The set of all such functions is denoted by~$T$.

\begin{definition} \label{def:bayes}
 The strategy
$\argmin\limits_{q_X \in Q_X} \sum\limits_{\theta \in \Theta} \tau(\theta) R_X(q_X,\theta)$
is called Bayesian with respect to the weight function~$\tau$. 
A strategy $q_X$ is called  Bayesian if a weight function~$\tau \in T$ exists with respect to which the strategy $q_X$ is Bayesian.
\end{definition}

By this definition, every Bayesian strategy is optimal in a sense that
it minimizes the Bayesian risk of recognizing a pseudo-complex object with a priori probability $\tau(\theta)$ of model~$\theta$. 
For a complex model, such treatment of Bayesian strategy is meaningless because the model is not random and its a priori probability distribution is not defined.
However, the concept of Bayesian strategies allows to partition the set of all strategies into the set of improper strategies and the set of all other strategies. 
This dichotomy is given by the following theorem, which states that every strategy is either Bayesian or improper.

\begin{theorem} \label{StrictlyBadStrategy}
	For every strategy $q_X^0 \in Q_X$, either a weight function $\tau ^* \in T$ exists such that
	\begin{equation} \label{qIsBayes}
		q_X^0=\argmin_{q_X \in Q_X}\sum_{\theta \in \Theta}\tau ^*(\theta) R_X(q_X,\theta)
	\end{equation}
	or a strategy $q_X^*$ exists such that the inequality
	\begin{equation} \label{qIsBad}
		R_X(q_X^*,\theta)<R_X(q_X^0,\theta)
	\end{equation}
	holds for every $\theta \in \Theta$. These two properties of strategies are incompatible.
\end{theorem}
\begin{proof} Let the function 
	$F{:}\ T \times Q_X \rightarrow \mathbb{R}$ be defined by
	\begin{equation}\nonumber 
		F( \tau, q_X)= \sum_{\theta \in \Theta} \tau(\theta)[R_X(q_X,\theta)-R_X(q_X^0,\theta)].
	\end{equation}
	Let the strategy $q_X^*$ and the weight function $\tau ^*$ be defined by
	\begin{equation} \nonumber 
		q_X^*= \argmin_{q_X \in Q_X} \max_{\tau \in T} F(\tau,q), \qquad
		\tau^*=\argmax_{\tau \in T} \min_{q_X \in Q_X}F(\tau,q_X).
	\end{equation}
	By definition (\ref{RiskDef}), the risk is a linear function of the probabilities $q_X(y'\given x)$, which form a strategy $q_X$. 
	Consequently, the function $F$ is also a linear function of~$q_X$ for any fixed weight function $\tau$. 
	With a fixed strategy $q_X$, the function $F$ is a linear function of weights $\tau (\theta)$. The set $T$ of all weight functions and the set $Q_X$ of all strategies are closed and convex. 
	Due to the well-known duality theorem \cite{borwein, boyd, hiriart}, for such functions $F$ we have
	\begin{equation} \nonumber
		\max_{\tau \in T} \min_{q_X \in Q_X}F(\tau,q_X) = F(\tau ^*,q_X^*) = \min_{q_X \in Q_X}\max_{\tau \in T} F(\tau,q_X).
	\end{equation}
	
	It is obvious that the equality $F(\tau,q_X^0)=0$ holds for every $\tau \in T$. Therefore the inequality
	$\min\limits_{q_X \in Q_X}F(\tau, q_X) \leq 0$ holds for every $\tau \in T$ and, consequently,
	\begin{equation} \nonumber
		\max_{\tau \in T} \min_{q_X \in Q_X}F(\tau,q_X) = F(\tau ^*,q_X^*) \leq 0 .
	\end{equation}
	Since the value $F(\tau ^*,q_X^*)$ is not positive, only two cases are possible: either $F(\tau ^*,q_X^*) < 0$ or
	$F(\tau ^*,q_X^*) = 0$. 
	Let us consider the case $F(\tau ^*,q_X^*) < 0$. We have the following chain:
	\begin{align*}
	F(\tau ^*, q_X^*) &= \min_{q_X \in Q_X}\max_{\tau \in T} F(\tau, q_X) 
                        = \max_{\tau \in T}F(\tau,q_X^*) =\\
			&= \max_{\tau \in T} \sum_{\theta \in \Theta} \tau (\theta)[R_X(q_X^*,\theta)-R_X(q_X^0,\theta)] 
                        = \max_{\theta \in \Theta}[R_X(q_X^*,\theta)-R_X(q_X^0,\theta)] .
	\end{align*}
		The chain results in the inequality 
			\begin{equation} \nonumber
			\max\limits_{\theta \in \Theta }[R_X(q_X^*,\theta)-R_X(q_X^0,\theta)] < 0.
				\end{equation} 
				Consequently, the inequality $R_X(q_X^*,\theta) < R_X(q_X^0,\theta)$
			holds for every model $\theta \in \Theta$ and property (\ref{qIsBad}) is proved. Note that property (\ref{qIsBayes}) does not hold in this case.
	
	Let us consider the case $F(\tau ^*,q_X^*) = 0$. We have the following chain:
	\begin{align*}
	F(\tau ^*,q_X^*) &= \max_{\tau \in T}\min_{q_X \in Q_X} F(\tau, q_X) 
                       = \min_{q_X \in Q_X} F(\tau ^*, q_X) = \\
                &= \min_{q_X \in Q_X} \sum_{\theta \in \Theta}\tau ^*(\theta)[R_X(q_X,\theta)-R_X(q_X^0,\theta)] = \\
 		        &= \min_{q_X \in Q_X}\biggl[\, \sum_{\theta \in \Theta}\tau  ^*(\theta)R_X(q_X,\theta)\biggr]-\sum_{\theta \in \Theta}\tau  ^*(\theta)R_X(q_X^0,\theta) .
	\end{align*}
	It implies the equality
	\begin{equation} \nonumber 
		\min_{q_X \in Q_X} \sum_{\theta \in \Theta}\tau  ^*(\theta)R_X(q_X,\theta) = \sum_{\theta \in \Theta}\tau  ^*(\theta)R_X(q_X^0,\theta)
	\end{equation}
	and thus property (\ref{qIsBayes}) is proved. Note that inequality (\ref{qIsBad}) does not hold in this case. 
\end{proof}

The theorem states that any reasonable strategy for complex object recognition has to minimize the weighted sum of risks and to be of the form
\begin{equation} \nonumber %
	q^*=\argmin_{q \in Q}
\sum_{\theta \in \Theta} \tau(\theta) R(q,\theta)
\end{equation}
	for some weights $\tau(\theta)$. Therefore, it has to make the decision of the form
\begin{equation} \label{BayesianStrategy} 
	y^*=\argmin_{y' \in Y}\sum_{y \in Y}\biggl[\, \sum_{\theta \in \Theta}\tau(\theta)p(x,y;\theta) \biggr] w(y,y')
\end{equation}
for a signal $x$.
For the special case of the loss function given by
\begin{equation} \label{LossFunction}
w(y,y') =
\begin{cases}
0 & \text{if}\;\; y=y' \\
1 & \text{if}\;\; y \neq y'
\end{cases} ,
\end{equation}
decision~(\ref{BayesianStrategy}) simplifies to
\begin{equation} \label{SimplifiedBayesianStrategy}
	y^*=\argmax_{y \in Y} \sum_{\theta \in \Theta}\tau(\theta) \, p(x,y;\theta) .
\end{equation}
The Theorem \ref{StrictlyBadStrategy} does not assign weight function $\tau$ that has to be used in (\ref{BayesianStrategy}) or (\ref{SimplifiedBayesianStrategy}) 
and so it does not uniquely define a strategy. 
Selecting an appropriate strategy depends on requirements that have to be satisfied in such or other applied situation. 
The theorem outlines only a strategy subset where one has to look for an appropriate strategy. 
However, the theorem results in important negative conclusion. 
It shows which strategies must not be used in any application. For some complex objects, commonly used maximum likelihood strategy that makes a decision
\begin{equation} \label{maxLikelyyhood} 
	y^*=\argmax_{y \in Y} \max_{\theta \in \Theta}p(x,y;\theta)
\end{equation}
occurs to be just such improper strategy.
%
There are quite ordinary and realistic cases when it is impossible to express the strategy (\ref{maxLikelyyhood}) in a form (\ref{SimplifiedBayesianStrategy}). 
       
\begin{example} \label{NonBayesianMLrecognition}
Let us consider a complex object with a set $\Theta=\{ 1,2 \}$ of models, a set $Y=\{ A,B \}$ of states and a set $X=\mathbb{R}^2$ of signals. 
A priori probabilities $p_Y(y=A)$ and $p_Y(y=B)$ of states are equal and do not depend on the model. 
For the fixed state and model, the signal $x = (x^{(1)}, x^{(2)})$ is a two-dimensional Gaussian random variable with independent components. 
The variances of both components $x^{(1)}$, $x^{(2)}$ are $\sigma^2=1$ and depend neither on the state nor on the model. 
Expected value $\mu (y, \theta)$ of the signal $x = (x^{(1)}, x^{(2)})$ depends both on the state and on the model so that 
$\mu(y=A, \theta = 1) = \mu(y=A, \theta = 2) = (0,0)$, $\mu(y=B, \theta = 1) = (0,1)$, $\mu(y=B, \theta = 2) = (1, 0)$. 
Probability density distributions $p_{X \given Y }(x|y; \theta)$ are shown on Figure \ref{figureAB} with circles. 

\begin{figure}[h!] 
		\center{\includegraphics[height=0.47\textwidth]{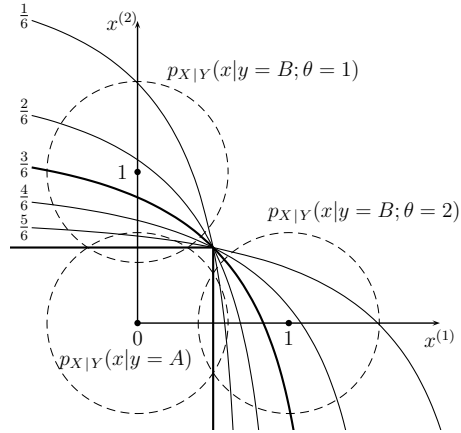}}
		\caption{ Family of different Bayesian strategies. }
		\label{figureAB}
\end{figure}

If a loss function is given by (\ref{LossFunction}) the risk $R_X(q_X, \theta)$ is a probability of a wrong decision made by strategy $q_X$ provided  $\theta $ is a true model. 
Let us assume such loss function and compare the maximum likelihood strategy (\ref{maxLikelyyhood}) with a family (\ref{SimplifiedBayesianStrategy}) of Bayesian strategies. 
Each strategy in the family (\ref{SimplifiedBayesianStrategy}) is specified by concrete values of weights $\tau(1)$ and $\tau(2)$. 

Maximum likelihood strategy (\ref{maxLikelyyhood}) makes the decision $y^* =A$ if the observed signal $x = (x^{(1)}, x^{(2)})$ belongs to the subset 
$X_A = \left\{(x^{(1)}, x^{(2)}) \in X| x^{(1)} \le 0.5, x^{(2)} \le 0.5 \right\}$ and makes the decision $y^* =B$ for signals from $X_B=X\setminus X_A$.  
The strategy is represented at Figure \ref{figureAB} with a border between $X_A$ and $X_B$. 
The border consists of two half-lines, horizontal and vertical. In a similar way, several Bayesian strategies from the family (\ref{SimplifiedBayesianStrategy}) are represented with curves. 
The strategies are built for values $ \frac{1}{6}, \frac{2}{6},\frac{3}{6},\frac{4}{6},\frac{5}{6}$ of the weight $\tau(1)$ and corresponding values of the weight $\tau(2)=1-\tau(1)$.

One can see that maximum likelihood strategy for recognition of the considered complex object differs from any Bayesian strategy and due to Theorem \ref{StrictlyBadStrategy} is improper. 
In turn, it means that some Bayesian strategy exists that strongly predominates maximum likelihood strategy for both models $\theta \in \Theta$. 
The Theorem \ref{StrictlyBadStrategy} does not say anything on how this better strategy should be selected from the family (\ref{SimplifiedBayesianStrategy}). 
However, one can guess that it is a Bayesian strategy with $\tau(1) = \tau(2) = 0.5 $. 
Really, if a true model is $\theta = 1$ the maximum likelihood strategy makes a wrong decision with probability $\approx 0.37$. 
The same is the probability of a wrong decision if $\theta = 2$. 
Bayesian strategy with  $\tau(1) = \tau(2) = 0.5 $ makes a wrong decision with probability $\approx 0.35$ both for $\theta = 1$ and for $\theta = 2$.      
\end{example}

We have excluded from consideration all obviously bad strategies, which left us only with Bayesian strategies. 
This does not mean that every Bayesian strategy predominates any non-Bayesian one. 
It only means that for each non-Bayesian strategy some better strategy exists and this better strategy is Bayesian. 
Moreover, it does not mean that each Bayesian strategy in itself is suitable for any application.
Selecting a particular strategy from the set of all Bayesian strategies depends on the requirements on the strategy. 

Perhaps, the most popular requirement in the theory and practice of complex hypothesis discrimination is the following so-called minimax requirement \cite{lehman}.  
For given sets $X$, $Y$, $\Theta$ and probabilities $p(x,y;\theta)$, $x \in X$, $y \in Y$, $\theta \in \Theta$, a strategy 
$q_X{:}\ Y \times X \rightarrow \mathbb{R}$ has to be found that minimizes~$c$ subject to the condition that
$R_X(q_X,\theta) \leq c$ for each  $\theta \in \Theta$ .
In other words, the strategy
\begin{equation} \label{MinMaxShort}
q_X^*= \argmin_{q_X \in Q_X} \max_{\theta \in \Theta}R_X(q_X,\theta)
 \end{equation}
is called a minimax strategy. Obviously, the minimax strategy is not improper and, consequently, is Bayesian. 
Indeed, if $q_X^*$ would be improper then such strategy $q_X'$ would exist that the inequalities 
\begin{equation} \nonumber 
R_X(q_X',\theta) < R_X(q_X^*,\theta), \quad \theta \in \Theta,
\end{equation}
would be valid as well as the inequality 
\begin{equation} \nonumber 
\max_{\theta \in \Theta}R_X(q_X',\theta) < \max_{\theta \in \Theta}R_X(q_X^*,\theta)
\end{equation}
that would contradict  (\ref{MinMaxShort}). Minimax strategies are fruitfully used in the theory and practice  of recognition without learning \cite{duda, webb}.

However, the minimax requirement is not the only possible meaningful one. 
In multi-criteria decision making \cite{zeleny1982} another reasonable requirement is used that differs from the minimax one. Let us follow the way shown in \cite{zeleny1982} and 
define a number $\min\limits_{q_X \in Q_X} R_X(q_X,\theta)$ for each $\theta \in \Theta$ and call it the risk of optimal strategy. 
In general, the risk of the optimal strategy depends on the model and so does the optimal strategy 
$\argmin\limits_{q_X \in Q_X} R_X(q_X,\theta)$. 
However, for some extraordinary convenient complex objects the strategy $\argmin\limits_{q_X \in Q_X} R_X(q_X,\theta)$ does not depend on the model.

\begin{definition} \label{def:optimal}
A strategy~$q_X^*$ satisfying the equality 
$R_X(q_X^*,\theta)=\min\limits_{q_X \in Q_X} R_X(q_X,\theta)$
for every model $\theta \in \Theta$ is called an overall optimal strategy.
\end{definition}
The existence of an overall optimal strategy is a rare exception rather than the rule. 
We formulate the problem of complex object recognition so that its solution is an overall optimal strategy when such strategy exists. 
Otherwise, the obtained strategy risk deviates from the optimal strategy risk as little as possible. 
More exactly,
for a given complex object $\left\langle X, Y, \Theta, p:X\times Y \times \Theta \rightarrow\mathbb{R}\right\rangle$ and a loss function $w{:}\ Y \times Y \rightarrow \mathbb{R}$, the strategy
$q_X{:}\ Y \times X \rightarrow \mathbb{R}$ has to be found that minimizes~$c$ subject to conditions
\begin{equation} \nonumber 
R_X(q_X,\theta)-\min_{q_X' \in Q}R_X(q_X',\theta) \leq c , \quad \theta \in \Theta .
\end{equation}
\begin{definition} \label{def:nearopt}
The strategy
$q_X^*=\argmin\limits_{q_X \in Q_X} \bigl[\max\limits_{\theta \in \Theta} \bigl(R_X(q_X,\theta)-\min\limits_{q_X' \in Q} R_X(q_X',\theta)\bigr)\bigr]$
is called a closest to optimal strategy.
\end{definition} 
Just as the minimax strategy the closest to optimal strategy is not improper and, consequently, is also Bayesian. 
The concept of the closest to optimal strategy is a straightforward application of the recommendations in \cite{zeleny1982} to our special case. 
On the other hand, the closest to optimal strategy is a generalization of the minimax deviation (regret) strategy from \cite{AlaizRodriguez}. 
The latter is a special case of the closest to optimal strategy when the statistical model is defined up to a priori probabilities of states.

The minimax approach~(\ref{MinMaxShort}) and the closest to optimal one~(\ref{def:nearopt}) result in different Bayesian strategies.
If the task is only recognition without learning, there is no reason to decide which one of these two approaches is better.  
However, the situation is different when the approaches are applied to recognition with learning.
In this case the minimax strategies are not suitable for recognition with learning at all. They simply ignore the learning sample. 
This defect of minimax approach has been detected by H.Robbins at the very beginning of his empirical Bayesian approach. 
This defect is considered in details in the next Section~\ref{EmpiricalBayesian}. 
Then in the Section~\ref{PatternRecognitionLearning}, we show that the closest to optimal strategies are free of this defect. 

\section{Empirical Bayesian approach and unsupervised learning}
\label{EmpiricalBayesian}
The principal property of a complex object is that its model is unknown but it is not random. It is fixed and does not change. 
The property gives rise to conjecture that recognizing a sample of independent and identically distributed objects can be performed 
better than isolated recognition of each element in the sample independently. 
The idea was researched by H.Robbins \cite{robbinsAssymptotical} decades ago and initiated an empirical Bayesian approach \cite{neyman2Breaks}. 
In fact, the same idea is the basis of an approach in pattern recognition known as unsupervised learning. 
However, these two approaches exist as if they arouse and were being developed independently of one another. 
We consider them as two different paths to the same goal and formulate a problem that may be considered as one of possible concretizations of empirical Bayesian approach and 
one of possible modifications of unsupervised learning.

H.Robbins explains his idea on the following simple example \cite{robbinsAssymptotical}, which we will use several times through the article.
An object is considered with a set $X=\mathbb{R}$ of signals, a set $Y=\{1,2\}$ of states and a set $\Theta=\{ \theta\mid 0\leq \theta \leq 1\}$ of models.
The object generates a Gaussian random signal $x \in X$, whose variance is~$1$, and whose mean equals~$1$ if the object is in state~1 and equals $(-1)$ if the object is in state~2. 
Only a priori probabilities of the states are unknown. 
They depend on the model so that the probability of state~$1$ is~$\theta$ and the probability of state~$2$ is $1-\theta$. 
In other words, the probability densities $p(x,y;\theta)$ that define the complex object are
\begin{equation} \label{RobbinsModel}
p(x,y;\theta)= \frac {\theta_y }{ \sqrt {2 \pi  } }{\rm e}^{ - \frac {1} {2} (x-\mu_y)^2 },
\end{equation}
where  $\theta_1=\theta$, $\theta_2=1-\theta$, $\mu_1=1$, $\mu_2=(-1)$. The loss function is given by~(\ref{LossFunction}). 
Let us consider optimal and minimax strategies
\begin{equation} \nonumber 
q_X^{\rm opt}(\theta)=\argmin_{q_X \in Q_X}R_X(q_X,\theta ) , \qquad
q_X^{\rm minmax}=\argmin_{q_X \in Q_X} \max_{\theta \in \Theta}  R_X(q_X,\theta)
\end{equation}
and see how the risks $R_X(q_X^{\rm opt}(\theta),\theta)$ and $R_X(q_X^{\rm minmax},\theta)$ depend on the model $\theta$ (see Figure~\ref{test1Bayesminmax}). 
Each strategy has a drawback compared to the other one. 
The strategy $q_X^{\rm opt}(\theta)$ can be used only if the model $\theta$ is known. 
The strategy $q_X^{\rm minmax}$ can be used even if the model $\theta$ is unknown, 
however the risk $R_X(q_X^{\rm minmax},\theta)$ is worse than $R_X(q_X^{\rm opt}(\theta),\theta)$ if $\theta$ is near to $1$  or $0$. 

\begin{figure}[h!] 
		\center{\includegraphics[height=0.47\textwidth]{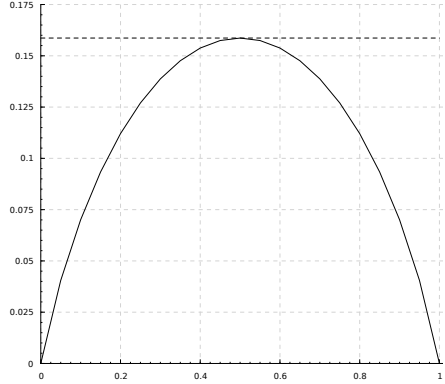}}
		\caption{The dependency of $R_X(q_X^{\rm minmax},\theta)$ on model $\theta$ is plotted with a dashed curve. 
        Probabilities $\min_{q_X \in Q_X}R_X(q_X,\theta)$ are plotted with a solid curve. }
		\label{test1Bayesminmax}
\end{figure}

Suppose that rather than a single object, a sample of independent random objects has to be recognized. 
The model $\theta$ according to which the objects are generated is unknown but fixed for the whole sample. 
H.Robbins formulates the main idea of empirical Bayesian approach that such a sample can be recognized much better than each of its elements independently.
His important negative result is that this improvement cannot be achieved in the framework of the minimax approach. Let us state this negative result exactly.

Let a complex object be defined by the sets $X=\mathbb{R}$, $Y=\{ 1,2 \}$, $\Theta=\{ \theta\mid0\leq \theta \leq 1\}$ and the probability density~(\ref{RobbinsModel}). 
Let the loss function have the form~(\ref{LossFunction}).
For this complex object and a positive integer~$n$, another complex object is defined, a so called compound complex object. 
It is characterized by sets $X^n$, $Y^n$, $\Theta$, probabilities
\begin{equation} \nonumber 
p^n(\bar{x},\bar{y};\theta)= \prod_{i=1}^{n}p(x_i,y_i;\theta) , \quad
\bar{x} \in X^n, \; \bar{y} \in Y^n, \; \theta \in \Theta,
\end{equation}
and losses
\begin{equation} \nonumber 
w^n(\bar{y},\bar{y'})=\frac{1}{n}\sum_{i=1}^nw(y_i,y_i'), \quad
\bar{y} \in Y^n, \; \bar{y'} \in Y^n.
\end{equation}
Let $q_{X^n}{:}\ Y^n \times X^n \rightarrow \mathbb{R}$ be the recognition strategy for the compound object and let
\begin{equation} \nonumber 
R_{X^n}(q_{X^n},\theta)=\sum_{\bar{y'} \in Y^n} \sum_{\bar{y} \in Y^n} \sum_{\bar{x} \in X^n} q_{X^n}(\bar{y'}\given \bar{x}) \, p^n(\bar{x},\bar{y};\theta) \, w^n(\bar{y},\bar{y'})
\end{equation}
be the risk of the strategy~$q_{X^n}$ with respect to the model~$\theta$.
The quoted negative result states that
\begin{equation} \nonumber
\min\limits_{q_{X^n}} \max\limits_{\theta \in \Theta}R_{X^n}(q_{X^n},\theta)=
\min\limits_{q_X} \max\limits_{\theta \in \Theta}R_X(q_X,\theta)
\end{equation}
for any number~$n$. It means that the quality of the strategy
\begin{equation} \nonumber
q_{X^N}^{\rm minmax}= \argmin\limits_{q_{X^n}} \max\limits_{\theta \in \Theta}R_{X^n}(q_{X^n},\theta)
\end{equation}
is not better than
\begin{equation} \nonumber
q_{X}^{\rm minmax}= \argmin\limits_{q_{X}} \max\limits_{\theta \in \Theta}R_{X}(q_{X},\theta).
\end{equation}
{\em Minimax strategies do not take advantage of the situation that multiple objects with a common model have to be recognized.\/}
We will show in Section \ref{PatternRecognitionLearning} that this drawback of minimax strategies also holds in a more general case. 

Despite of this clearly negative result, H.Robbins presents a heuristic strategy that for the objects of the form~(\ref{RobbinsModel}) 
and observed signals $(x_1,x_2, \ldots ,x_n)$ makes the decisions
\begin{equation} \label{RobbinsHeur}
y_i = 
\begin{cases}
1 & \text{if}\;\; x_i \ge \alpha \\
2 & \text{if}\;\; x_i < \alpha
\end{cases}
\quad\text{where}\quad
\alpha =  {1\over2} \ln \frac{n-\sum_{i=1}^{n}x_i} {n+\sum_{i=1}^{n}x_i} .
\end{equation}
The strategy (\ref{RobbinsHeur}) does not predominate $q_X^{\rm opt}(\theta)$ because $q_X^{\rm opt}(\theta)$ is optimal by definition. 
The strategy (\ref{RobbinsHeur}) does not predominate $q_{X^n}^{\rm minmax}$ because $q_{X^n}^{\rm minmax}$ is Bayesian and no strategy predominates it. 
However, the strategy (\ref{RobbinsHeur}) can be used even when the model $\theta$ is unknown. 
From this point of view it is better than the strategy~$q_X^{\rm opt}(\theta)$ that makes the decision
\begin{equation} \label{qOPTTheta}
y_i = 
\begin{cases}
1 & \text{if}\;\; x_i \ge \alpha \\
2 & \text{if}\;\; x_i < \alpha
\end{cases}
\quad\text{where}\quad
\alpha =  {1\over2} \ln \frac{1-\theta} {\theta} 
\end{equation} 
that depends on $\theta$.
The strategy~(\ref{RobbinsHeur}) is better than $q_{X^n}^{\rm minmax}$ in a sense that its quality improves as $n$~increases while the quality of $q_{X^n}^{\rm minmax}$ 
does not depend on $n$ and remains equal to the quality of 
$q_{X}^{\rm minmax}$. Moreover, for large enough~$n$ the strategy~(\ref{RobbinsHeur}) behaves almost 
like~$q_X^{\rm opt}(\theta)$.

The strategy~(\ref{RobbinsHeur}) illustrates the main idea of the empirical Bayesian approach \cite{neyman2Breaks, robbinsAssymptotical, robbinsEmpirical}, 
namely that an object sample can be recognized  much better 
than by recognizing each object independently. However, the example says nothing about the requirement that the strategy~(\ref{RobbinsHeur}) should be derived from. 
It only makes clear that it is not the minimax one. The problem of compound object recognition is not stated as precisely as the similar problems in the Bayesian or minimax decision theory. 
Therefore it is unclear how the approach should be formulated in general, not only for the special case~(\ref{RobbinsModel}). 

One of the possible specifications of the empirical Bayesian approach may be the following problem definition of compound complex object recognition. 
The input data of this problem consist of a quadruple 
$\left\langle X, Y, \Theta, p:X\times Y \times \Theta \rightarrow\mathbb{R}\right\rangle$ representing a complex object, 
a loss function $w{:}\ Y \times Y \rightarrow \mathbb{R}$ and a positive integer $n$.
These data define risks 
\begin{equation} \nonumber 
\min_{q_X \in Q_X}R_X(q_X,\theta)=
\min_{q_X \in Q_X}\sum_{y' \in Y} \sum_{y \in Y} \sum_{x \in X} q_X(y'\given x) \, p(x,y;\theta)\, w(y,y')
\end{equation} 
of optimal strategies, a compound object $\left\langle X^n, Y^n, \Theta, p^n:X^n\times Y^n \times \Theta \rightarrow\mathbb{R}\right\rangle$ and a loss function
$w^n{:}\ Y^n \times Y^n \rightarrow \mathbb{R}$ where
\begin{equation} \nonumber %
p^n(\bar{x},\bar{y};\theta)= \prod_{i=1}^{n}p(x_i,y_i;\theta) , \quad
\bar{x} \in X^n, \; \bar{y} \in Y^n, \; \theta \in \Theta,
\end{equation}
\begin{equation} \nonumber %
w^n(\bar{y},\bar{y'})=\frac{1}{n}\sum_{i=1}^n w(y_i,y_i'), \quad
\bar{y} \in Y^n , \; \bar{y'}\in Y^n.
\end{equation}
Let $q_{X^n}{:}\ Y^n \times X^n \rightarrow \mathbb{R}$ be a randomized strategy that makes a decision about hidden states $y_1, y_2, \cdots , y_n$ 
on the base of observed signals $y_1, y_2, \cdots , y_n$. 
Let $Q_{X^n}$ be the set of all such strategies. For $q_{X^n} \in Q_{X^n}$ and model $\theta \in \Theta$, the risk is a number
\begin{equation} \nonumber %
R_{X^n}(q_{X^n},\theta)=\sum_{\bar{y'} \in Y^n} \sum_{\bar{y} \in Y^n} \sum_{\bar{x} \in X^n} q_{X^n}(\bar{y'}\given \bar{x}) \, p^n(\bar{x},\bar{y};\theta) \, w^n(\bar{y},\bar{y'}).
\end{equation}
The problem of compound complex object recognition is defined so that for a quadruple 
\begin{equation} \nonumber %
\left\langle X, Y, \Theta, p:X\times Y \times \Theta \rightarrow\mathbb{R}\right\rangle,
\end{equation}
a loss function $w{:}\ Y \times Y \rightarrow \mathbb{R}$ and a positive integer $n$ the closest to optimal strategy
\begin{equation} \label{TheMainFormulation}
q_{X^n}^*= \argmin_{q_{X^n} \in Q_{X^n}} \max_{\theta \in \Theta} \bigl[ R_{X^n}(q_{X^n},\theta)-\min_{q_X \in Q_X}R_X(q_X,\theta) \bigr]
\end{equation}
has to be found. Obviously, the solution $q_{X^n}^*$ to this problem is a Bayesian strategy. 


Widely known unsupervised learning is another approach that copes with recognizing compound complex objects. It proceeds as follows. 
First, for given sets~$X$, $Y$, $\Theta$, probabilities $p(x,y;\theta)$, $x \in X$, $y \in Y$, $\theta \in \Theta$, and a sample $\bar x =(x_1,x_2, \ldots ,x_n)$ 
the maximum likelihood estimate
\begin{equation} \label{MaximalLikelyhoodModel}
\theta^{\rm ML} = \argmax_{\theta \in \Theta} \prod_{i=1}^n \sum_{y \in Y} p(x_i,y;\theta)=\argmax_{\theta \in \Theta} \sum_{i=1}^n \log \sum_{y \in Y} p(x_i,y;\theta)
\end{equation}
is found. Then, the elements of the sample are recognized using the strategy
\begin{equation} \label{MaximalLikelyhoodStrategy}
q_X^{\rm ML} = \argmin_{q_X \in Q_X} R_X(q_X,\theta^{\rm ML})
\end{equation}
as if the obtained model $\theta^{\rm ML}$ would be a true model. 

The maximum likelihood estimation~(\ref{MaximalLikelyhoodModel}) is of the main interest here because it is not trivial even in the simplest cases.
To solve~(\ref{MaximalLikelyhoodModel}), the algorithms~\cite{schlezEM} are used, which later became widely known 
as the EM-algorithms~\cite{dempsterEM}. The problem (\ref{MaximalLikelyhoodModel}) is significant in its own right.
Nevertheless, for some complex objects the strategy~(\ref{MaximalLikelyhoodModel}, \ref{MaximalLikelyhoodStrategy})
is improper even though each strategy of the family 
\begin{equation} \label{LokBayesianStrategy}
q_X^* = \argmin_{q_X \in Q_X} R_X(q_X,\theta^*), \quad \theta^* \in \Theta, 
\end{equation}
is Bayesian. Each strategy of the family (\ref{LokBayesianStrategy}) differs from the strategy (\ref{MaximalLikelyhoodModel}), (\ref{MaximalLikelyhoodStrategy}) 
at least due to the fact that their formats differ from one another, the first being $Y \times X \rightarrow \mathbb{R}$ and the second $Y^n \times X^n \rightarrow \mathbb{R}$. 
However, even if $n=1$ and their formats become identical, some complex objects exist such that  
the strategy (\ref{MaximalLikelyhoodModel}), (\ref{MaximalLikelyhoodStrategy}) differs from each strategy of the family (\ref{LokBayesianStrategy}). 
Moreover, it differs from every Bayesian strategy.

\begin{example} \label{NonBayesianUnsupervised}
\begin{figure}[h!] 
		\center{\includegraphics[height=0.47\textwidth]{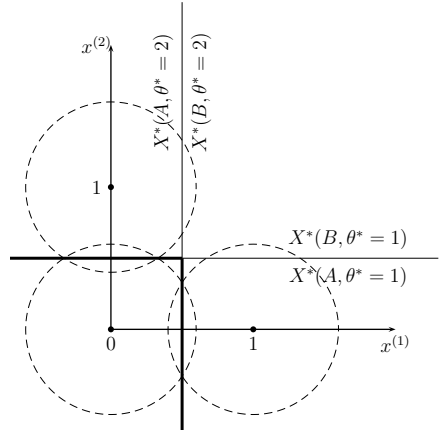}}
		\caption{ Maximum likelihood strategy. }
		\label{figureABNonBayes}
\end{figure}
Let us consider a complex object and a loss function as in the Example \ref{NonBayesianMLrecognition} and compare the unsupervised learning strategy (\ref{MaximalLikelyhoodModel}), 
(\ref{MaximalLikelyhoodStrategy}) with strategies of the family (\ref{LokBayesianStrategy}) for the simplest case when $n=1$. 
For a given signal $x=(x^{(1)},x^{(2)})$ the strategy (\ref{LokBayesianStrategy}) makes a decision 
\begin{equation} \label{LokBayesianDecisionEx}
y^*=\argmax_{y \in Y}p(x^{(1)},x^{(2)},y; \theta^*),
\end{equation}
that depends on $\theta^*$. For $\theta^* =1 $ the strategy $q_X^*$ divides the plane $X=\mathbb{R}^2$ into two half-planes
$X^*(A,\theta^*=1)$ and $X^*(B,\theta^*=1)$ with a horizontal line with coordinate $x^{(2)}=0.5$. 
For $\theta^* =2 $ it divides $X$ into $X^*(A,\theta^*=2)$ and $X^*(B,\theta^*=2)$ with a vertical line with coordinate $x^{(1)}=0.5$ (see Figure \ref{figureABNonBayes}).
The strategy (\ref{MaximalLikelyhoodModel}), (\ref{MaximalLikelyhoodStrategy}) makes a decision
\begin{equation} \label{UnsupervisedDecisionEx}
y^{\rm ML}=\argmax_{y \in Y}p(x^{(1)},x^{(2)},y; \argmax_{\theta \in \Theta}\sum_{y \in Y} p(x^{(1)},x^{(2)},y;\theta))
\end{equation}
that does not depend on $\theta$. 
This decision is $A$ on subset 
\begin{equation} \nonumber
\{ (x^{(1)},x^{(2)}) \in X | x^{(1)} \le 0.5, x^{(2)} \le 0.5 \} = X^*(A,\theta^*=1) \cap X^*(A,\theta^*=2)
\end{equation}
that differs both from $X^*(A,\theta^*=1)$ and from $X^*(A,\theta^*=2)$. 
Therefore, the strategy (\ref{MaximalLikelyhoodModel}), (\ref{MaximalLikelyhoodStrategy}) differs both from $\argmin\limits_{q_X \in Q_X} R_X(q_X,\theta^*=1)$ 
and from  $\argmin\limits_{q_X \in Q_X} R_X(q_X,\theta^*=2)$. 
It has already been shown in Example \ref{NonBayesianMLrecognition} that the strategy differs from any Bayesian strategy.
\end{example}

The heuristic strategy~(\ref{RobbinsHeur}) by H.Robbins, the closest to optimal strategy~(\ref{TheMainFormulation}) and the unsupervised learning strategy~(\ref{MaximalLikelyhoodModel}, \ref{MaximalLikelyhoodStrategy})
are three different strategies. 
They are compared experimentally for the complex model~(\ref{RobbinsModel}) in Example~\ref{RobbinsExample} in Section~\ref{Testing}.
In this experiment, {\em the closest to optimal strategy yields the best result, the unsupervised learning strategy the worst, and the heuristic strategy an intermediate one.\/}
%

\section{Recognition with learning}
\label{PatternRecognitionLearning}
As in the previous sections, we consider complex objects,  defined by a quadruple \\
$\left\langle X, Y, \Theta, p:X\times Y \times \Theta \rightarrow\mathbb{R}\right\rangle$.
In addition, a certain source of data is considered. The source generates the so called learning information $z$ that belongs to a finite set $Z$. 
The learning information~$z$ is random and depends on the model $\theta \in \Theta$, so that the probability $p(z;\theta)$ is defined for every $z \in Z$ and $\theta \in \Theta$. 
It is crucial that for a fixed model~$\theta$, the learning information~$z$  depends neither on the current state~$y$ nor on the current signal~$x$ of the object, 
so that $p(z,x,y;\theta)=p(z;\theta) p(x,y;\theta)$. Formally, the source of learning information is defined with a triple 
$\left\langle Z, \Theta, p:Z \times \Theta \rightarrow \mathbb{R} \right\rangle$. 

Usually, learning information is obtained in a following way.  
The object is put into a special condition so that even its state becomes  observable. 
Under this condition, a sample $(x_1,y_1; x_2,y_2; \ldots \; x_n,y_n)$  of~$n$ independent random pairs $(x,y)$ 
is observed so that $Z=(X \times Y)^n$ and
\begin{equation} \nonumber %
p(z;\theta)=p(x_1,y_1, x_2,y_2, \ldots , x_n,y_n;\theta)=\prod_{i=1}^n p(x_i,y_i;\theta).
\end{equation}
Our formulation is more general. In case of unsupervised learning, the learning information~$z$ is an element of $X^n$. 
When the learning information is absent, we introduce a special element $0$ and define $Z=\{0 \}$ and $p(z=0;\theta)=1$ for every model $\theta \in \Theta$.
The learning information can have a different nature, for example, it can be an expert opinion about the model.
Moreover, the learning information can come from several independent sources simultaneously as it is mentioned in \cite{Kittler98oncombining}. 
In this  case, the set $Z_s$ and the probabilities $p(z;\theta,s)$, $z \in Z_s$, $\theta \in \Theta$, $s \in S$, are given for every source $s \in S$.

The problem of recognition with learning means 
that for a given complex object and a learning information source a strategy has to be developed that makes a decision about current object state based both on the 
learning information and on the current signal generated by the object. We express this statement in the form
\begin{eqnarray} 
& \left\langle X, Y, \Theta, p:X\times Y \times \Theta \rightarrow\mathbb{R}\right\rangle, \left\langle Z, \Theta, p:Z \times \Theta \rightarrow \mathbb{R} \right\rangle  \notag \label{WithLearning}\\
& \downarrow \\
& q_{X \times Z}:Y \times X \times Z \rightarrow \mathbb{R}. \notag
\end{eqnarray}

The value $q_{X \times Z}(y'|x,z)$ of the strategy $q_{X \times Z}$ is a conditional probability of the decision "a state of an object is $y'$" 
under condition that a signal $x$ is observed and learning information $z$ is available. The form
\begin{eqnarray} 
& \left\langle X, Y, \Theta, p:X\times Y \times \Theta \rightarrow\mathbb{R}\right\rangle \notag \\  
& \downarrow \label{WithoutLearning} \\
& q_{X}:Y \times X \rightarrow \mathbb{R}. \notag
\end{eqnarray} 
expresses the problem of recognition without learning. One can see that the form (\ref{WithLearning}) is a special case of (\ref{WithoutLearning}). Really, (\ref{WithLearning}) can be equivalently expressed as 
\begin{eqnarray}
& \left\langle X^*, Y, \Theta, p^*:X^*\times Y \times \Theta \rightarrow\mathbb{R}\right\rangle  \notag \\
& \downarrow \notag \\
& q_{X^*}:Y \times X^* \rightarrow \mathbb{R} \notag
\end{eqnarray} 
with $X^*=X \times Z$ and $p^*(x^*,y; \theta)=p^*(x,z,y;\theta)=p(z;\theta)\cdot p(x,y;\theta)$. 
In other words, a complex object with a source of learning information can be replaced with another object that generates specific signals composed of two components. 
The first component depends both on the current state of the object and on the model. 
The second one depends directly only on the model and under a fixed model depends neither on the current state nor on the first component.

Theorem~\ref{StrictlyBadStrategy}, proved in Section~\ref{SecComplexHypothesis} in the general form, remains valid for recognition with learning. 
However, for subsequent considerations another formulation of the theorem is more convenient.
A strategy $q_{X \times Z}:Y \times X \times Z \rightarrow \mathbb{R}$ will be represented as a superposition of two functions: 
a learning procedure $g_Z{:}\  Z \rightarrow Q$ and a recognition strategy $ q_X{:}\ Y \times X \rightarrow \mathbb{R}$. 
Given learning information~$z$, the learning procedure~$g_Z$ defines the recognition strategy $g_Z(z){:}\ Y \times X \rightarrow R$.
This means that after learning with random information~$z$ and observing a signal~$x$, the decision "the object is in state~$y'$" is made with probability $g_Z(z)(y'\given x)$. 
Let $G_Z$~be the set of all possible learning procedures of the form $Z \rightarrow Q$.
 
As before, $R_X(q_X,\theta)$ is the risk of recognition strategy $q_X \in Q_X$ with respect to the model $\theta \in \Theta$. 
The number $R_X(g_Z(z),\theta)$ is the risk of the strategy $g_Z(z)$, obtained by applying the learning procedure $g_Z{:}\  Z \rightarrow Q$ to the learning information $z \in Z$. 
The risk $R_X(g_Z(z),\theta)$ is random because it depends on the random information~$z$. Let $R_Z(g_Z,\theta)$ denote the expectation of the risk over the set of all learning informations,
\begin{equation} \nonumber 
R_Z(g_Z,\theta)= \sum_{z \in Z} p(z;\theta)R_X(g_Z(z),\theta).
\end{equation}
\begin{definition}
A learning procedure~$g_Z^*$ is called Bayesian if a weight function $\tau{:}\  \Theta \rightarrow \mathbb{R}$ exists such that
\begin{equation} \nonumber 
g_Z^* = \argmin_{g_Z \in G_Z} \sum_{\theta \in \Theta} \tau(\theta)R_Z(g_Z,\theta).
\end{equation}
\end{definition}
Now the restriction of Theorem~\ref{StrictlyBadStrategy} 
to learning procedures reads as follows.
\begin{theorem} \label{StrictlyBadProcedure}
	For every learning procedure $g_Z^0 \in G$, either a weight function $\tau ^* \in T$ exists such that
	\begin{equation} \nonumber 
		g_Z^0=\argmin_{g_Z \in G_Z}\sum_{\theta \in \Theta}\tau ^*(\theta) R_Z(g_Z,\theta)
	\end{equation}
	or a learning procedure~$g_Z^*$ exists such that for every model $\theta \in \Theta$ the inequality
	\begin{equation} \nonumber 
		R_Z(g_Z^*,\theta)<R_Z(g_Z^0,\theta)\end{equation}
		holds. These two properties of learning procedures are incompatible.
\end{theorem}

The theorem says that every learning procedure is either Bayesian or improper. In the special case when a sample
 $( x_1,y_1; x_2,y_2; \ldots; x_n,y_n) $ is the learning information  and a signal $x_0$ is observed, the decision~$y^*$ about the current object state has to be of the form
 \begin{equation} \label{CorrectDecision}
y^*= \argmin_{y' \in Y} \sum_{y_0 \in Y} \sum_{\theta \in \Theta} \tau(\theta)\prod_{i=0}^n p(x_i,y_i;\theta) w(y_0,y')
\end{equation}
with some weights $\tau(\theta)$, $\theta \in \Theta$. For the loss function~(\ref{LossFunction}), the strategy is simplified to
\begin{equation} \nonumber 
y^*= \argmax_{y_0 \in Y}  \sum_{\theta \in \Theta} \tau(\theta)\prod_{i=0}^n p(x_i,y_i;\theta).
\end{equation}  
Strategies that are commonly used in pattern recognition have other form. 
The most widely known is the method based on maximum likelihood estimation \cite{duda, webb}. According to this method, the model
\begin{equation} \label{SupervisedLearning1}
\theta^{\rm ML}=\argmax_{\theta \in \Theta}\prod_{i=1}^n p(x_i,y_i;\theta)
\end{equation} 
is found and then for an observed signal~$x$ the decision     
\begin{equation} \label{SupervisedLearning2}
y^{\rm ML}=\argmin_{y' \in Y} \sum_{y \in Y} p(x,y;\theta^{\rm ML}) w(y,y')
\end{equation} 
is made, as if the maximum likelihood model was identical to the true one. For some complex objects, the strategy~(\ref{SupervisedLearning1}, \ref{SupervisedLearning2})
differs from~(\ref{CorrectDecision}) and thus it is improper. One of such complex objects was considered in Examples 
\ref{NonBayesianMLrecognition} and \ref{NonBayesianUnsupervised}. The following example shows more sharply the main defect of maximum likelihood learning (\ref{SupervisedLearning1}), (\ref{SupervisedLearning2}). 

\begin{example} \label{MLRobbins}
Let us return to complex object 
$\langle  X,Y,\Theta, p:X \times Y \times \Theta \rightarrow \mathbb{R} \rangle$ described in a Section \ref{EmpiricalBayesian} with sets $X=\mathbb{R}$, $Y=\{ 1,2 \}$, $\Theta = \{ \theta| 0 \le \theta \le 1 \}$ and probability density distribution $p: X \times Y \times \Theta \rightarrow \mathbb{}R$ of the form (\ref{RobbinsModel}). The loss function is given by~(\ref{LossFunction}).
If no learning sample is available one can use a minimax recognition strategy that does not depend on a model. The strategy decides that $y=1$ if $x \ge 0$ and $y=2$ if $x < 0$. The strategy makes a wrong decision with probability $\approx 0.16$ independently of a model and seems to be quite reasonable. 

Suppose now that the learning sample is obtained that consists of only pair $(x_0, y_0)$. It is very small sample but one can expect that it enables at least small improvement as compared with minimax strategy that works if no learning sample is available. However, the maximum likelihood learning (\ref{SupervisedLearning1}), (\ref{SupervisedLearning2}) results in a strategy that ignores a current signal $x$ at all and decides that the object is in the state $y_0$ independently of signal generated by object. So, if the learning sample is of very small size it is reasonable to ignore it rather than to make use of it with maximum likelihood learning.     
\end{example}\\
Nevertheless, the maximum likelihood learning~(\ref{SupervisedLearning1}, \ref{SupervisedLearning2})
has an important positive property that can be called asymptotic optimality. 
Asymptotic optimality qualifies a learning procedure in a similar sense as consistency qualifies a parameter estimate.
We formulate and prove this property after several preliminary remarks. 

First, we assume that for every two different models~$\theta_1$ and~$\theta_2$ a pair $(x,y) \in X \times Y$ exists such that $p(x,y;\theta_1) \neq p(x,y;\theta_2)$. 
Obviously, this assumption is not a restriction. 
Really, if there are two models $\theta_1 $ and $\theta_2 $ such that $p(x,y;\theta_1) = p(x,y;\theta_2)$ for each pair $(x,y) \in X\times Y$ they can be treated as equivalent and be replaced with one model.

Then, we rely on properties of entropy-like functions $\sum\limits_{i=1}^n a_i \log x_i$ of $2n$ strictly positive arguments such that $\sum\limits_{i=1}^n a_i=\sum\limits_{i=1}^n x_i=1$. 
However, we extrapolate such functions onto the set of non-negative numbers assuming that $a\log x = 0$ if $a=x=0$ and $a \log x = -\infty $ if $a > 0$ and $x=0$. Under these assumptions the difference
\begin{equation} \nonumber 
\sum_{x \in X} \sum_{y \in Y} p(x,y;\theta^*) \log p(x,y; \theta')
-\sum_{x \in X} \sum_{y \in Y} p(x,y;\theta^*) \log p(x,y; \theta^*)
\end{equation} 
is defined for every pair $(\theta^*, \theta')$ and is strictly negative if  $\theta^* \neq \theta'$. The statement results from the known property of entropy-like functions that  
\begin{equation} \nonumber 
\max_{x_1} \max_{x_2} \ldots \max_{x_n} \sum\limits_{i=1}^n a_i \log x_i = \sum\limits_{i=1}^n a_i \log a_i  
\end{equation}  
and if $x_{i^*} \neq a_{i^*}$ for some $i^*$ then $\sum\limits_{i=1}^n a_i \log a_i > \sum\limits_{i=1}^n a_i \log x_i$.

At last, we use a fundamental relation between the averaged value $\frac{1}{n}\sum\limits_{i=1}^nx_i$ and the expected value 
$\sum\limits_{x \in X}x \cdot p(x)$. 
The relation restricted to our special case has the following formulation.
Let $n$ be a positive number, $X$ be a finite set of numbers, $p(x)$ be the probability of $x \in X$ and
$Z_n=\{z \in X^n|\frac{1}{n}\sum\limits_{i=1}^nx_i \ge 0\}$. If $\sum\limits_{x \in X}x \cdot p(x) < 0$ then
\begin{equation} \label{Bernulli}
\lim\limits_{n \rightarrow \infty}\sum\limits_{z \in Z_n}\prod\limits_{i=1}^np(x_i) =0.
\end{equation}
\begin{theorem} \label{MaxLikeConsistency}
Let $Z=(X \times Y)^n$ and the procedure $g_n^{\rm ML}{:}(X \times Y)^n \rightarrow Q$ be defined by
\begin{equation} \label{MaxLikeProcedure} 
g_n^{\rm ML}(z)=\argmin_{q_X \in Q_X}R_X(q_X, \argmax_{\theta \in \Theta}p(z;\theta)), \quad z \in (X \times Y)^n .
\end{equation}
Then for every $\theta^* \in \Theta$
\begin{equation} \label{BernulliMain} 
\lim_{n \rightarrow \infty}[R_Z(g_n^{\rm ML},\theta^*) - \min_{q_X \in Q_X}R_X(q_X,\theta^*) ] = 0. 
\end{equation}
\end{theorem}
\begin{proof}
Let us choose an arbitrary model $\theta^*$ and fix it for the whole proof. 
Let us denote
\begin{equation} \nonumber 
\theta^{\rm ML}(z) = \argmax_{\theta \in \Theta} p(z;\theta)= \argmax_{\theta \in \Theta}\prod_{i=1}^n p(x_i,y_i;\theta). 
\end{equation} 
The model $\theta^{\rm ML}(z)$ is random because it depends on a random sample $z$. If the probability of the sample $z$ is
\begin{equation} \nonumber 
p(z;\theta^*)=p(x_1,y_1, x_2,y_2, \ldots x_n,y_n; \theta^*) = \prod_{i=1}^n p(x_i,y_i;\theta^*)
\end{equation}
then inequality $\theta^{\rm ML}(z) \neq \theta^* $ is a random event with its probability depending on $n$. 
We prove that the probability of inequality $\theta^{\rm ML}(z) \neq \theta^* $ converges to zero when $n$ increases.
Let us define three subsets of the set $Z=(X \times Y)^n$:
\begin{equation} \nonumber
Z(\theta^*,\theta',n)=\{ z|p(z;\theta') \ge p(z;\theta^*)  \}, \quad 
Z(\theta^*,n)=\bigcup_{\theta' \in \Theta\setminus \{\theta^* \} }Z(\theta^*,\theta',n),
\end{equation}
\begin{equation} \nonumber
Z^{\rm ML}(\theta^*,n)=\{ z|\theta^{\rm ML}(z) \neq \theta^*   \}
\end{equation}
and denote $P(Z') = \sum\limits_{z \in Z'}p(z;\theta^*)$ for each $Z' \subset Z$. Evidently,
\begin{equation} \nonumber
Z^{\rm ML}(\theta^*,n) \subset Z(\theta^*,n)=\bigcup_{\theta' \in \Theta\setminus \{\theta^* \} }Z(\theta^*,\theta',n)
\end{equation}
that results in
\begin{equation} \label{Inequalities}
P(Z^{\rm ML}(\theta^*,n)) \le P(Z(\theta^*,n)) \le \sum_{\theta' \in \Theta\setminus \{\theta^* \} }P(Z(\theta^*,\theta',n)).
\end{equation}
The subset $Z(\theta^*,\theta',n)$ consists of all samples $z=(x_1,y_1, x_2,y_2, \ldots x_n,y_n)$ that satisfy the inequality 
\begin{equation} \nonumber 
\frac{1}{n}\sum_{i=1}^n[\log p(x_i,y_i;\theta')-\log p(x_i,y_i;\theta^*)] \ge 0.
\end{equation}
The left side of the inequality is the average of $n$ independent and identically distributed random numbers 
$\log p(x,y;\theta')-\log p(x,y;\theta^*)$. The mean value  
\begin{equation} \nonumber 
\sum_{x \in X} \sum_{y \in Y} p(x,y;\theta^*) \log p(x,y; \theta')
-\sum_{x \in X} \sum_{y \in Y} p(x,y;\theta^*) \log p(x,y; \theta^*)
\end{equation} 
of these numbers is negative and therefore $\lim\limits_{n \rightarrow \infty}P(Z(\theta^*,\theta',n)=0$ due to (\ref{Bernulli}) and\\ 
$\lim\limits_{n \rightarrow \infty}P(Z^{\rm ML}(\theta^*,n))~=~0$ due to (\ref{Inequalities}).
The proof is concluded by the chain
\begin{multline*}
\lim_{n \rightarrow \infty}[R_Z(g_n^{\rm ML},\theta^*) - \min_{q_X \in Q_X}R_X(q_X,\theta^*) ] = \\
\begin{aligned}
&=\lim_{n \rightarrow \infty}\sum_{z \in (X \times Y)^n}p(z;\theta^*)
[R_X(g_n^{\rm ML}(z),\theta^*) - \min_{q_X \in Q_X}R_X(q_X,\theta^*) ] \\
&=\lim_{n \rightarrow \infty}\sum_{\substack{ z \in (X \times Y)^n \\[1pt] \theta^{\rm ML}(z)  \neq \theta^*}}
p(z;\theta^*)
[R_X(g_n^{\rm ML}(z),\theta^*) - \min_{q_X \in Q_X}R_X(q_X,\theta^*) ] \\
&\leq\lim_{n \rightarrow \infty}\sum_{\substack{ z \in (X \times Y)^n \\[1pt] \theta^{\rm ML}(z)  \neq \theta^*}}
p(z;\theta^*)
[ \max_{y \in Y}\max_{y' \in Y} w(y,y') - \min_{y \in Y}\min_{y' \in Y} w(y,y')] \\
&=\lim_{n \rightarrow \infty}[ \max_{y \in Y}\max_{y' \in Y} w(y,y') - \min_{y \in Y}\min_{y' \in Y} w(y,y')]
\sum_{\substack{ z \in (X \times Y)^n \\[1pt] \theta^{\rm ML}(z)  \neq \theta^*}}p(z;\theta^*)\\
&=[ \max_{y \in Y}\max_{y' \in Y} w(y,y') - \min_{y \in Y}\min_{y' \in Y} w(y,y')]\lim_{n \rightarrow \infty} 
P(Z_n^{\rm ML}(\theta^*,n)) =0.
\end{aligned}
\end{multline*}
\end{proof}

%

We restrict our further considerations to Bayesian learning procedures, focusing in particular on the minimax procedure
$\argmin\limits_{g_Z \in G_Z} \max\limits_{\theta \in \Theta} R_Z(g_Z,\theta)$
and the closest to optimal procedure
$\argmin\limits_{g_Z \in G_Z} \max\limits_{ \theta \in \Theta} \bigl[ R_Z(g_Z,\theta)-\min\limits_{q_X \in Q_X}R_X(q_X,\theta) \bigr]$.
Both procedures belong to the Bayesian class.
However, minimax procedures have a fundamental drawback that they sometimes do not make use of a learning sample, no matter how large it is.
For a rather wide class of complex objects, the minimax learning procedure simply ignores the learning sample.
\begin{theorem} \label{ThMinMaxIsBad}
	Let for a complex object $\left\langle X, Y, \Theta, p:X\times Y \times \Theta \rightarrow\mathbb{R}\right\rangle$ such a model $\theta^*$ and a strategy $q_X^*$~exist that
	\begin{equation} \nonumber 
		q_X^*= \argmin_{q_X \in Q_X}R_X(q_X,\theta^*), \quad \theta^*= \argmax_{\theta \in \Theta}R_X(q_X^*,\theta). 
	\end{equation} 
	Then any source $\left\langle Z, \Theta, p:Z \times \Theta \rightarrow \mathbb{R} \right\rangle$ of learning information and any learning procedure 
	$g_Z{:}\  Z \rightarrow Q$ satisfy the inequality
	\begin{equation} \label{MaxMinEqualsMinMax}
		\max_{\theta \in \Theta}R_Z(g_Z,\theta) \geq \max_{\theta \in \Theta} R_X(q_X^*,\theta).	
	\end{equation} 
\end{theorem}
\begin{proof} For any learning procedure~$g_Z$, we have the chain
\begin{eqnarray}
& \max\limits_{\theta \in \Theta} R_Z(g_Z,\theta) \geq R_Z(g_Z,\theta^*) = \sum\limits_{z \in Z} p(z;\theta^*) R_X(g_Z(z), \theta^*) \geq \notag \\
& \geq \sum\limits_{z \in Z}p(z;\theta^*) R_X(q_X^*, \theta^*)= R_X(q_X^*,\theta^*) = \max\limits_{\theta \in \Theta}R_X(q_X^*,\theta). \notag
\end{eqnarray}

\end{proof}
The theorem shows that there are complex objects for which the minimax approach is particularly inappropriate. 
Inequality~(\ref{MaxMinEqualsMinMax}) states that any learning procedure, however sophisticated it may be, is useless from the minimax point of view. 
It cannot yield a recognition strategy that would be better than some strategy that does not use the learning sample at all.
The following examples show that such situations are nothing unusual.

\begin{example}
\label{ex:1}
	The conditions of Theorem~\ref{ThMinMaxIsBad} are satisfied if for every weight function $ \tau\in T$, the model set~$\Theta$ contains a model~$\theta'$ such that 
    $p(x,y;\theta')=\sum\limits_{\theta \in \Theta}\tau(\theta) p(x,y;\theta)$. 
	One of such model sets is just the set (\ref{RobbinsModel}), for which H.Robbins shows that the minimax strategy does not achieve the result that the empirical Bayesian approach does.
\end{example}

\begin{example}
	Let $X=\mathbb{R}^2$,  $Y=\{1,2\}$, $p_Y(y=1)=p_Y(y=2)=0.5$, and $p_{X\given Y}(x\given y)$ be a two-dimensional Gaussian probability density distribution with the unit covariance matrix. 
	The expected values $\theta_y$, $y \in \{1,2\}$, of the random signal $x \in X$ are unknown. 
    It is only known that they belong to given closed convex disjoint sets~$\Theta_1$ and~$\Theta_2$. 
	This means, the set of models is $\Theta=\Theta_1\times\Theta_2$.
	This model set does not satisfy the condition of Example~\ref{ex:1}
	but it satisfies the conditions of Theorem~\ref{ThMinMaxIsBad}. 
	In this case, the model $\theta^* \in \Theta$ is the pair
	\begin{equation} \nonumber
		(\theta_1^*,\theta_2^*)= \argmin_{\theta_1 \in \Theta_1, \theta_2 \in \Theta_2} (\theta_1-\theta_2)^2,
	\end{equation}
	and the strategy~$q^*$ makes the decision
	\begin{equation} \nonumber
		y^* =
                \begin{cases}
                1 & \text{if}\;\; (x-\theta_1^*)^2 \leq (x-\theta_2^*)^2,\\
                2 & \text{if}\;\; (x-\theta_1^*)^2 > (x-\theta_2^*)^2.
                \end{cases}
	\end{equation}
	It is an acceptable recognition strategy for the complex object considered. 
	However, it is clear that a better strategy could be found if a learning sample was available, in particular if this sample was large enough. 
	Nevertheless, the minimax strategy ignores this opportunity even if the learning sample is very large. 
\end{example}

\begin{example}
	Let $p_y{:}\  X \rightarrow \mathbb{R}$,  $y \in \{1,2\}$, be two functions whose values are conditional probabilities $p_{X \given Y}(x\given y)$, so that $p_y(x)=p_{X \given Y}(x\given y)$. 
    It is a typical situation in pattern recognition that the functions~$p_1$ and~$p_2$ 
	are unknown and only a set~$P$ is known that contains both~$p_1$ and~$p_2$. 
	In this case, the set of models is $\Theta=P\times P$, 
	the model~$\theta^*$ is any pair $p_1=p_2$, and $q_X^*$ is the strategy 
	that for every signal $x$ decides about the state $y$ with equal probabilities $q_X^*(y=1\given x)=q_X^*(y=2\given x)$. This strategy is quite bad but no better strategy is possible.
	However, it is clear, at least intuitively, that the situation could be improved if a learning sample was available. 
    Minimax strategies are not able to take advantage of a learning sample, no matter how large.
\end{example}

Theorem \ref{ThMinMaxIsBad} is a severe negative result. The result is quite transparent and easy to prove, especially, 
if one notices that the condition of the theorem states that the function $R_X:Q_X \times \Theta \rightarrow \mathbb{}R$ has a saddle point and the pair $(q_X^*, \theta^*)$ is this point. 
Probably, the property was discovered many times before. 
In a special case, this property was noticed by H.Robbins when he wanted to derive the empirical Bayesian approach from the ideas of minimax decisions \cite{robbinsAssymptotical}.
This clarity may have been the reason for the hasty and pessimistic conclusion that recognition with learning is something very different from recognition without learning, 
something beyond the statistical decision theory. 

Theorems \ref{MaxLikeConsistency} and \ref{ThMinMaxIsBad} show the state of today's knowledge of complex object recognition as well as a gap in this knowledge. 
The minimax approach is suited for recognizing hidden state based on the signal that depends both on the state and on the model. 
However, for some complex objects this approach ignores learning information that depends directly only on the model and does not depend on the state. 
Unlike minimax approach, maximum likelihood learning makes use of the learning sample. 
It allows to recognize complex objects almost as if the model of the object is known. 
However, it is possible only when the learning sample size can be increased almost infinitely. 
For limited samples and for some objects maximum likelihood learning is not Bayesian and is predominated by some other methods. 
Therefore, today's developers 
have to cope alone with the problem when the model of an object is not completely known but several few examples of object's behaviour (let us say, 4-5) are available.

Let us see now how the closest to optimal learning procedures behave. 
As shown by the following theorem similar to Theorem~\ref{ThMinMaxIsBad}, these procedures are also useless for a certain class of complex objects.

\begin{theorem} \label{SubBayesIsBad}
Let for a complex object $\left\langle X, Y, \Theta, p:X\times Y \times \Theta \rightarrow\mathbb{R}\right\rangle$  a model $\theta^*$ and a strategy $q_X^*$~exist such that
	\begin{equation} \label{NearOptStrategy}
			q_X^*= \argmin_{q_X \in Q_X}[R_X(q_X,\theta^*)-\min_{q_X' \in Q_X}R_X(q_X',\theta^*)],
		\end{equation} 
	\begin{equation} \label{NearOptModel}
			\theta^* = \argmax_{\theta \in \Theta}[R_X(q_X^*,\theta)-\min_{q_X' \in Q_X}R_X(q_X',\theta)].
		\end{equation} 
	Then any source $\left\langle Z, \Theta, p:Z \times \Theta \rightarrow \mathbb{R} \right\rangle$ of learning information and any learning procedure 
	$g_Z{:}\  Z \rightarrow Q$ satisfy the inequality
	\begin{equation} \nonumber 
		\max_{\theta \in \Theta}[R_Z(g_Z,\theta)-\min_{q_X' \in Q_X}R_X(q_X',\theta)] \geq \max_{\theta \in \Theta}
	[R_X(q_X^*,\theta)-\min_{q_X' \in Q_X}R_X(q_X',\theta)] .
	\end{equation}
\end{theorem}

However, the consequences of this theorem for closest to optimal learning are not so destructive as those of Theorem~\ref{ThMinMaxIsBad} for minimax learning. 
In fact, conditions~(\ref{NearOptStrategy}) and~(\ref{NearOptModel}) imply that an overall optimal Bayesian strategy exists for the recognized object, namely the strategy~$q_X^*$. 
Really, it follows from (\ref{NearOptStrategy}) that
	\begin{equation} \nonumber
		q_X^*= \argmin_{q_X \in Q_X}[R_X(q_X,\theta^*)-\min_{q_X' \in Q_X}R_X(q_X',\theta^*)]
		= \argmin_{q_X \in Q_X}R_X(q_X,\theta^*),
	\end{equation} 
	\begin{equation} \nonumber
	R_X(q_X^*,\theta^*)-\min_{q_X' \in Q_X}R_X(q_X',\theta^*)=0.
	\end{equation} 
Thus, a condition  (\ref{NearOptModel}) obtains a form
 \begin{equation} \nonumber 
 		0 \geq R_X(q_X^*,\theta)-\min_{q_X' \in Q_X}R_X(q_X',\theta)  \text{   for all   } \theta \in \Theta.
 \end{equation}
Evidently,  
\begin{equation} \nonumber 
 		R_X(q_X^*,\theta)-\min_{q_X' \in Q_X}R_X(q_X',\theta) \ge 0  \text{  for all  } \theta \in \Theta,
\end{equation}
and so 
\begin{equation} \nonumber 
 		R_X(q_X^*,\theta)=\min_{q_X' \in Q_X}R_X(q_X',\theta)  \text{   for all  } \theta \in \Theta.
\end{equation}
In this case, any learning approach is useless, not only closest to optimal one. 

Just as maximum likelihood learning, closest to optimal learning is also asymptotically optimal in the sense of the following theorem.
\begin{theorem} \label{SubBayes Is Consistent}
Let $Z=(X \times Y)^n$ and the learning procedure $g_n^*{:}\ (X \times Y)^n \rightarrow Q$ be defined by
\begin{equation} \nonumber 
g_n^*=\argmin_{g_Z \in G_Z} \max_{\theta \in \Theta} [R_Z(g_Z, \theta)- \min_{q_X \in Q_X}R_X(q_X,\theta)].
\end{equation}
Then
\begin{equation} \label{eq:SubBayes Is Consistent}
\lim_{n \rightarrow \infty}\max_{\theta \in \Theta}[R_Z(g_n^*,\theta) - \min_{q_X \in Q_X}R_X(q_X,\theta) ] = 0. 
\end{equation}
\end{theorem}    
\begin{proof} 
We proved earlier (see (\ref{BernulliMain}))  that
\begin{equation} \nonumber 
\lim_{n \rightarrow \infty}[R_Z(g_n^{\rm ML},\theta) - \min_{q_X \in Q_X}R_X(q_X,\theta) ] = 0
\end{equation}
for procedure~(\ref{MaxLikeProcedure}) and every model~$\theta$. Consequently,
\begin{equation} \label{ConvergToZero} 
\lim_{n \rightarrow \infty}\max_{\theta \in \Theta}[R_Z(g_n^{\rm ML},\theta) - \min_{q_X \in Q_X}R_X(q_X,\theta) ] = 0
\end{equation}
because the number $|\Theta|$ of models is finite. Clearly,
\begin{equation} \label{evident}
\min_{g_Z \in G_Z} \max_{\theta \in \Theta}[R_Z(g_Z,\theta) - \min_{q_X \in Q_X}R_X(q_X,\theta) ] \leq 
\max_{\theta \in \Theta}[R_Z(g_n^{\rm ML},\theta) - \min_{q_X \in Q_X}R_X(q_X,\theta) ] .
\end{equation} 
Equality (\ref{ConvergToZero}) and inequality~(\ref{evident}) imply~(\ref{eq:SubBayes Is Consistent}).
\end{proof}
Hence, the concept of closest to optimal learning fills the above-mentioned gap in complex object recognition. 
It embraces the whole range of learning sample sizes, from zero to infinity. 
Unlike minimax approach it ignores the learning sample only if an overall optimal strategy exists for the recognized object. 
Certainly, in this case no information that specifies the model is needed. 
Similarly to the risks of maximum likelihood learning, the risks of closest to optimal learning converge to the risks of the optimal strategy. 
Unlike maximum likelihood learning, closest to optimal learning is Bayesian and is predominated by no other learning procedure.

\section{Developing closest to optimal learning procedures}
\label{SubBayesianLearning} 
We consider a complex object, presented with a quadruple $\left\langle X, Y, \Theta, p:X\times Y \times \Theta \rightarrow\mathbb{R}\right\rangle,$ 
and a source of learning information, presented with a triple $\left\langle Z, \Theta, p:Z \times \Theta \rightarrow \mathbb{R} \right\rangle$.
The following two theorems show the way of developing the closest to optimal learning
procedure. Clearly, this procedure is predominated with no procedure and therefore it is Bayesian. The next theorem specifies a weight function, with respect to which it is Bayesian.
\begin{theorem}  \label{ThSubOptimalProcedure}
	The closest to optimal learning procedure
	\begin{equation} \nonumber
	g_Z^*=\argmin_{g_Z \in G_Z} \max_{ \theta \in \Theta} \bigl[ R_Z(g_Z,\theta)-\min_{q_X \in Q_X}R_X(q,\theta) \bigr]
	\end{equation}
	is a Bayesian procedure with respect to weights $\tau^*(\theta), \theta \in \Theta$ that maximize the difference
	\begin{equation} \nonumber 
		\biggl[\min_{g_Z \in G_Z}\sum_{\theta \in \Theta}\tau(\theta)\cdot R_Z(g_Z,\theta)\biggr]-\biggl[\sum_{\theta \in \Theta}\tau(\theta)\cdot \min_{q_X \in Q_X}R_X(q_X,\theta)\biggr].
	\end{equation}
\end{theorem}
\begin{proof}
	Let us define a function $F{:}\ T\times G_Z\to\mathbb{R}$ given by
	\begin{equation} \nonumber 
		F(\tau,g_Z) = \biggl[\sum_{\theta \in \Theta}\tau(\theta)\cdot R_Z(g_Z,\theta)\biggr]-\biggl[\sum_{\theta \in \Theta}\tau(\theta)\cdot \min_{q_X \in Q_X}R_X(q_X,\theta)\biggr]
	\end{equation}
that allows to express procedure $g_Z^*$ and weight function $\tau^*$ in a form
	\begin{equation} \nonumber 
		g_Z^*=  \argmin_{g_Z \in G_Z} \max_{\tau \in T}F(\tau,g_Z), \quad 
		\tau^*= \argmax_{\tau \in T} \min_{g_Z \in G_Z}F(\tau, g_Z).
	\end{equation}
	We have to prove that $g_Z^*= \argmin\limits_{g_Z \in G_Z}\sum\limits_{\theta \in \Theta}\tau^*(\theta)R_Z(g_Z, \theta)$.
	The function $F{:}\ T\times G_Z \rightarrow \mathbb{R}$ is a linear function of $g_Z$ for every fixed $\tau$ and a linear function of $\tau$ for every fixed $g_Z$. 
    It is defined on the Cartesian product of two convex closed sets~$T$ and~$G_Z$. Due to the well-known duality theorem \cite{borwein, boyd, hiriart} this function satisfies the equality
	\begin{equation} \nonumber 
		\min_{g_Z \in G_Z}\left[ \max_{\tau \in T}F(\tau, g_Z)\right]= F(\tau^*,g_Z^*)=\max_{\tau \in T}\left[\min_{g_Z \in G_Z}F(\tau, g_Z)\right].
	\end{equation}
	This implies the equality $F(\tau^*, g_Z^*) = \min\limits_{g_Z \in G_Z}F(\tau^*, g_Z)$
 and the chain
	\begin{align*}
		g_Z^* &= \argmin_{g_Z \in G_Z}F(\tau^*,g_Z) 
               = \argmin_{g_Z \in G_Z}\biggl[\sum_{\theta \in \Theta}\tau^*(\theta) R_Z(g_Z,\theta)-\sum_{\theta \in \Theta}\tau^*(\theta) \min_{q_X \in Q_X}R(q,\theta)\biggr] = \\
		&= \argmin_{g_Z \in G_Z}\sum_{\theta \in \Theta}\tau^*(\theta) R_Z(g_Z,\theta).
	\end{align*}
\end{proof}
For each $g_Z \in G_Z$, the function $F(\tau, g_Z)$ depends linearly on weights $\tau(\theta)$, $\theta \in \Theta$. 
Therefore, evidently, the function $\Phi(\tau)=\min_{g_Z \in G_Z}F(\tau, g_Z)$ is a concave function of these weights.
Nevertheless, we deduce this fact directly from the definition of concavity and obtain an explicit expression for the supergradient of the function $\Phi$, 
which is necessary for its further maximization.
\begin{theorem} \label{ThSubOptimalConcave}
	The function $\Phi{:}\ \mathbb{R}^{\Theta} \rightarrow \mathbb{R}$ with values
	\begin{equation} \nonumber
		\Phi(\tau)=\biggl[\min_{g_Z \in G_Z}\sum_{\theta \in \Theta}\tau(\theta)\cdot R_Z(g_Z,\theta)\biggr]-\biggl[\sum_{\theta \in \Theta}\tau(\theta)\cdot \min_{q_X \in Q_X}R_X(q_X,\theta)\biggr]
	\end{equation}
	is a concave function of weights $\tau(\theta)$, $\theta \in \Theta$. 
\end{theorem}
\begin{proof}
	We show that for every point $\tau^* \in \mathbb{R}^{\Theta}$ some linear function $L{:}\ \mathbb{R}^{\Theta} \rightarrow \mathbb{R}$ exists that satisfies the inequality 
	\begin{equation} \label{supergradient}
	L(\tau - \tau^*) \geq \Phi(\tau) - \Phi(\tau^*)
	\end{equation}
	for every $\tau \in \mathbb{R}^{\Theta}$. Note that the gradient of this linear function is a supergradient of the function $\Phi$ at the point $\tau^*$. 
    One of possible linear functions that satisfy (\ref{supergradient}) is the function with values
	\begin{equation} \nonumber
		L(\tau)=\sum_{\theta \in \Theta}\tau(\theta)\biggl[ R_Z(g_Z^*,\theta)- \min_{q_X \in Q_X}R_X(q_X,\theta)\biggr]
	\quad\text{where}\quad
		g_Z^*= \argmin_{g_Z \in G_Z} \sum_{\theta \in \Theta} \tau^*(\theta)R_Z(g_Z,\theta).
	\end{equation}
	Indeed, 
\begin{multline*}
L(\tau - \tau^*)= \\
\begin{aligned}
&= \sum_{\theta \in \Theta}\tau(\theta)\biggl[ R_Z(g_Z^*,\theta)- \min_{q_X \in Q_X}R_X(q_X,\theta)\biggr]-
	\sum_{\theta \in \Theta}\tau^*(\theta)\biggl[ R_Z(g_Z^*,\theta)- \min_{q_X \in Q_X}R_X(q_X,\theta)\biggr] \geq \\
&\geq \phantom{-} \min_{g_Z \in G_Z}\sum_{\theta \in \Theta}\tau(\theta) \biggl[   R_Z(g_Z,\theta)- \min_{q_X \in Q_X}R_X(q_X,\theta) \biggr] -\\
		&\phantom{\geq} - \min_{g_Z \in G_Z}\sum_{\theta \in \Theta}\tau^*(\theta) \biggl[ R_Z(g_Z,\theta)- \min_{q_X \in Q_X}R_X(q_X,\theta)\biggr] 
= \Phi(\tau)-\Phi(\tau^*).
\end{aligned}
\end{multline*}
\end{proof}
Due to concavity of the function $\Phi$, it can be maximized using well-known methods of convex optimization, such as supergradient hill-climbing \cite{shorMethBook}.
To do that, an auxiliary procedure to project a point $\tau \in \mathbb{R}^\Theta$ onto the set $T$ is needed. 
Strictly speaking, this procedure finds the point $\argmin_{\tau' \in T}(\tau' - \tau)^2$ for a given point $\tau \in \mathbb{R}^\Theta$. In our case the procedure is as follows:
\begin{enumerate} \itemsep0pt
\item for all $\theta \in \Theta$  do $\tau'(\theta) = \tau(\theta);$
\item for all $\theta \in \Theta$ do $ \tau'(\theta)= \tau'(\theta)-(\sum_{\theta \in \Theta}\tau'(\theta)-1) \cdot |\Theta|^{-1};$
\item if $\tau'(\theta)\geq 0$ for all $\theta \in \Theta$ then $STOP$;
\item for all $\theta \in \Theta$ do $\tau'(\theta)= \max \{ \tau'(\theta), 0\};$
\item go to 2;
\end{enumerate}

Next we present a technique to compute the closest to optimal learning procedure. 
We use the word `technique' rather than `algorithm', to emphasize that some of its steps are far from trivial and have to be developed specifically for every task at hand. 
The technique can be seen as an algorithm only in very simple cases, such as the examples given in Section~\ref{Testing}.
 
Initially, the following input data have to be given: 
\begin{enumerate} \itemsep1pt
\item sets $X$, $Y$, $\Theta$ and probabilities $p(x,y;\theta)$, $x \in X$, $y \in Y$, $\theta \in \Theta$, which define the complex object to be recognized; 
\item set $Z$ and probabilities $p(z;\theta)$, $z \in Z$, $\theta \in \Theta$, which define the source of learning information; 
\item values of the loss function $w(y,y')$ for every $y \in Y, y' \in Y$
\item numbers $\min_{q_X \in Q_X}  R_X(q_X,\theta)$ for every $\theta \in \Theta$;
\item sequence of numbers $\gamma_i$, $i = 1,2, \ldots$, such that \\$\lim_{i \rightarrow \infty}\gamma_i=0$ and 
$\lim_{n \rightarrow \infty} \sum_{i=1}^n \gamma_i=\infty$; 
\item required accuracy $\varepsilon > 0$; 
\item initial weights $\tau(\theta)$, $\theta \in \Theta$; for example,  $\tau(\theta)= |\Theta|^{-1}$; 
\item initial values of numbers $S$ and $s$, which will keep an upper and lower bound of the quality of the learning procedure; for example $S=\infty$ and $s=0$.
\end{enumerate}
The $i$-th iteration of the technique proceeds as follows:
\begin{enumerate} \itemsep1pt
\item construct a Bayesian procedure $g_Z$ for the current weight function $\tau$; 
\item for every model $\theta \in \Theta$ calculate numbers $\Delta(\theta) = R_Z(g_Z,\theta)-\min_{q_X \in Q_X}  R_X(q_X,\theta)$; they form the gradient of function $\Phi$ at the current point $\tau$; 
\item calculate $\Delta=\max_\theta \Delta(\theta)$; 
it is the quality of the current procedure $g_Z$;
\item calculate
$\bar{\Delta}=\sum_\theta \tau(\theta)\Delta(\theta)$; it is the value of $\Phi(\tau)$ at the current point $\tau$; 
\item update the bounds as $S=\min\{S,\Delta\}$ and $s=\max\{s,\bar\Delta\}$; 
\item if $S-s<\varepsilon$ then $STOP$;
\item set the weights $\tau(\theta)$, $\theta \in \Theta$, to the projection of the weights $\tau(\theta)+\gamma_i \Delta(\theta)$ on $T$; 
\end{enumerate}
Supergradient optimization, which iteratively updates the initial weights, is performed only in steps 2, 3 and 7. 
The other steps calculate the data used in the stopping condition, which, as will be shown, guarantees the required accuracy.

Let $S$ and $s$ be the bounds after termination of the algorithm. It means that a learning procedure $g_Z'$ with the quality 
\begin{equation} \label{TheBestQ}
\max_{\theta \in \Theta} [R_Z(g_Z',\theta)-\min_{q_X \in Q_X}  R_X(q_X,\theta)]=S 
\end{equation}
has been obtained at some iteration. It also means that a weight function $\tau'$ with
\begin{equation} \label{the BestTau}
\Phi(\tau')=\biggl[\min_{g_Z \in G_Z}\sum_{\theta \in \Theta}\tau'(\theta)\cdot R_Z(g_Z,\theta)\biggr]-\biggl[\sum_{\theta \in \Theta}\tau'(\theta)\cdot \min_{q_X \in Q_X}R_X(q_X,\theta)\biggr]=s
\end{equation}
has been obtained at some (possibly different) iteration. 
As shown by the following theorem, the procedure $g_Z'$ and the weight function $\tau'$ have the property that the difference between the quality of the obtained procedure 
and the best procedure is not greater than $\varepsilon$.

\begin{theorem} \label{Exactness}
	The learning procedure $g_Z'$ satisfies the inequality
	\begin{equation} \nonumber 
		\max_{\theta \in \Theta} [R_Z(g_Z',\theta)-\min_{q_X \in Q_X}  R_X(q_X,\theta)]-
		\min_{g_Z \in G_Z}\max_{\theta \in \Theta} [R_Z(g_Z,\theta)-\min_{q_X \in Q_X}  R_X(q_X,\theta)]<\varepsilon.
	\end{equation}
\end{theorem}
\begin{proof}
	As before, we consider the function
	\begin{equation} \nonumber 
		F(\tau,g_Z)=\sum_{\theta \in \Theta}\tau(\theta)\bigl[  R_Z(g_Z,\theta)-\min_{q_X \in Q_X}R_X(q_X,\theta)\bigr].
	\end{equation}
	By (\ref{TheBestQ}) and (\ref{the BestTau}), for the procedure $g_X'$ and the weight function $\tau'$ we have the chain
	\begin{align*}
		s&=\min_{g_Z \in G_Z}\sum_{\theta \in \Theta}\tau'(\theta)\bigl[  R_Z(g_Z,\theta)-\min_{q_X \in Q_X}R_X(q_X,\theta)\bigr]
		=\min_{g_Z \in G_Z} F(\tau',g_Z) \\
                &\leq\max_{\tau \in T}\min_{g_Z \in G_Z }F(\tau,g_Z)
		=\min_{g_Z \in G_Z}\max_{\tau \in T}F(\tau,g_Z) \\
                &=\min_{g_Z \in G_Z}\max_{\tau \in T}\sum_{\theta \in \Theta}\tau(\theta)\bigl[  R_Z(g_Z,\theta)-\min_{q_X \in Q}R_X(q_X,\theta)\bigr] \\
		&=\min_{g_Z \in G_Z}\max_{\theta \in \Theta}[  R_Z(g_Z,\theta)-\min_{q_X \in Q_X}R_X(q_X,\theta)\bigr]\\
                &\leq\max_{\theta \in \Theta}[  R_Z(g_Z',\theta)-\min_{q_X \in Q_X}R_X(q_X,\theta)\bigr]=S .
	\end{align*}
	From this chain and the stopping condition $S-s<\varepsilon$, it follows that 
	\begin{equation} \nonumber
		\max_{\theta \in \Theta} [R_Z(g_Z',\theta)-\min_{q_X \in Q_X}  R_X(q_X,\theta)]-
		\min_{g_Z \in G_Z}\max_{\theta \in \Theta} [R_Z(g_Z,\theta)-\min_{q_X \in Q_X}  R_X(q_X,\theta)]<\varepsilon.
	\end{equation}
\end{proof}
\section{Examples}
\label{Testing}

 For several simplest cases, 
 the closest to optimal learning procedure $g^0$ is compared with the maximum likelihood learning procedure $g^{ML}$ and with another learning procedures $g^H$ (to be defined later).  
 In all the examples, we have an object with sets $X=\mathbb{R}$, $Y=\{1,2\}$ and probability density distribution of the form
 \begin{equation}\nonumber
 	p(x,y;\theta)=p_y \cdot (\sqrt{2\pi})^{-1}\cdot {\rm e}^{-\frac{1}{2}(x-\mu_y)^2}
 \end{equation}
 with some unknown parameters specified later.  
 Such object is a complex object in a sense of Definition \ref{DefComplex} and the compared learning procedures for it can be implemented almost exactly. 
 We use the loss function~(\ref{LossFunction}). The required accuracy for constructing closest to optimal procedures is $\varepsilon=0.01$.

\begin{example*}\label{bondarenkoExample}
	\cite{bondarenkoEng} Let $\mu_1=1$,  $\mu_2=(-1)$, and only the a priori probabilities $p_y$, $y \in \{1,2\}$, be unknown. 
	Thus, the set of models is $\Theta=\{ \theta\mid0 \leq \theta \leq 1 \} $ and the a priori probabilities $p_y$ of the states are $p_1=\theta$, $p_2=1-\theta$. 
	This is the example used by H.Robbins to explain his idea of empirical Bayesian approach. 
	We simplify the example even further by assuming that the learning information is not  a signal sample $(x_1,x_2, \ldots ,x_n)$ but a state sample $(y_1,y_2, \ldots ,y_n)$. 
    To use the technique from Section~\ref{SubBayesianLearning} the set $\Theta$ is discretized with a step $0.05$.
	
	Figure~\ref{test1Bayesminmax} (page \pageref{test1Bayesminmax}) shows how the probabilities $\min_{q \in Q}R(q,\theta)$ and  $R(q^{\rm minmax},\theta)$ depend on model $\theta$. 
	Note that they do not depend on the sample size $n$ because they do not depend on the sample at all.
	Figure~\ref{test1} shows how the risks $R_G(g^{ML},\theta)$ and $R_G(g^0,\theta)$ depend on the model $\theta$ for the sample sizes $n=1,2,4$.
    
	
	It is striking how high the risk $R_G(g^{ML},\theta)$ is for some models. 
	Of course, one cannot expect the risk to be too low because the learning samples are very small and thus they bring little information about the true model. 
    Nevertheless, however small the amount of information, it is non-zero and it should be used to improve subsequent recognition, at least to some extent. 
    But the example shows that it is much better to ignore this information than to use it with maximum likelihood learning. 
    The example illustrates quite transparently the main drawback of maximum likelihood learning.
\end{example*}

    \begin{figure}[!h] 
		\begin{tabular}{c c}
			\includegraphics[height=0.35\textwidth]{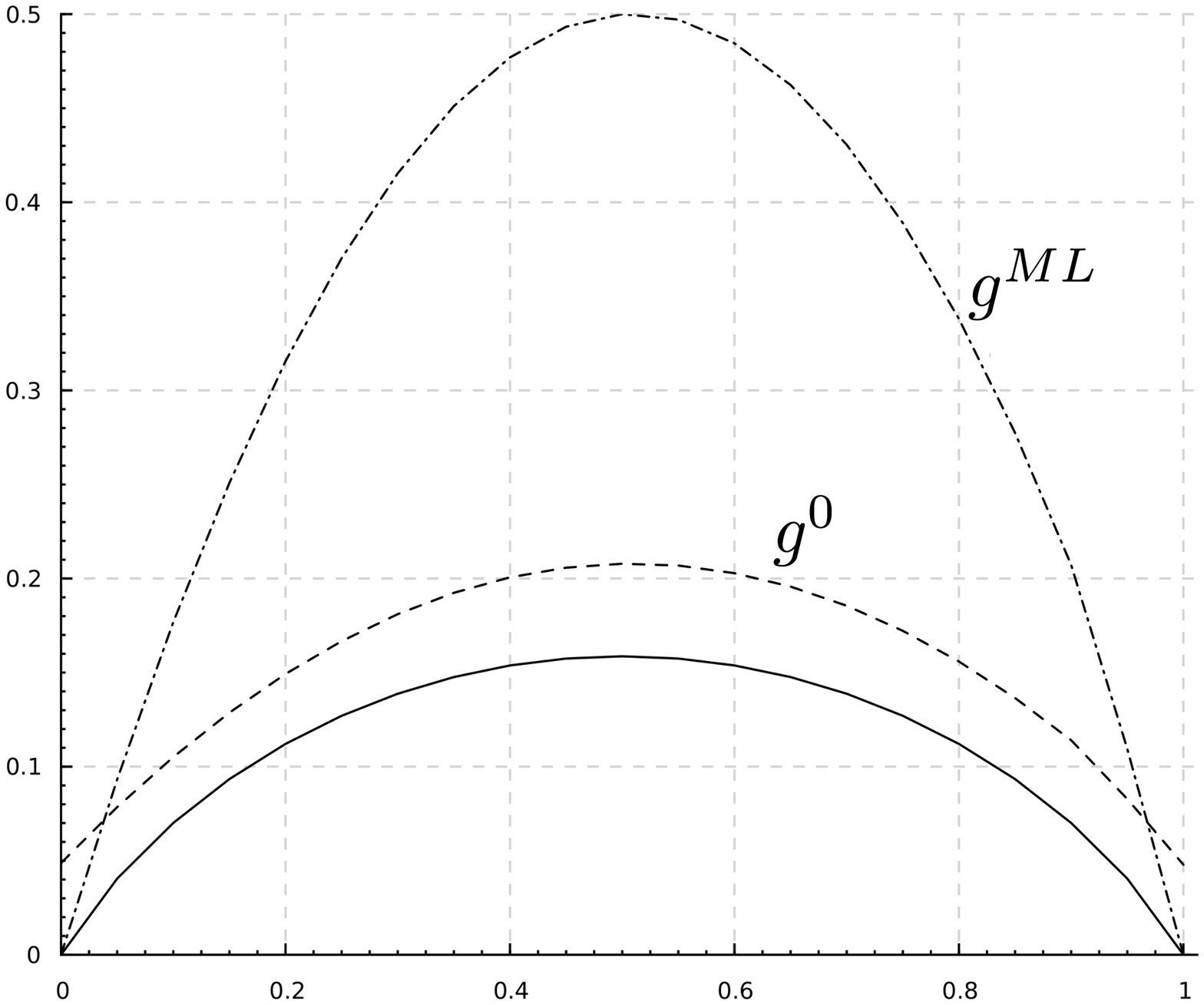} & \includegraphics[height=0.35\textwidth]{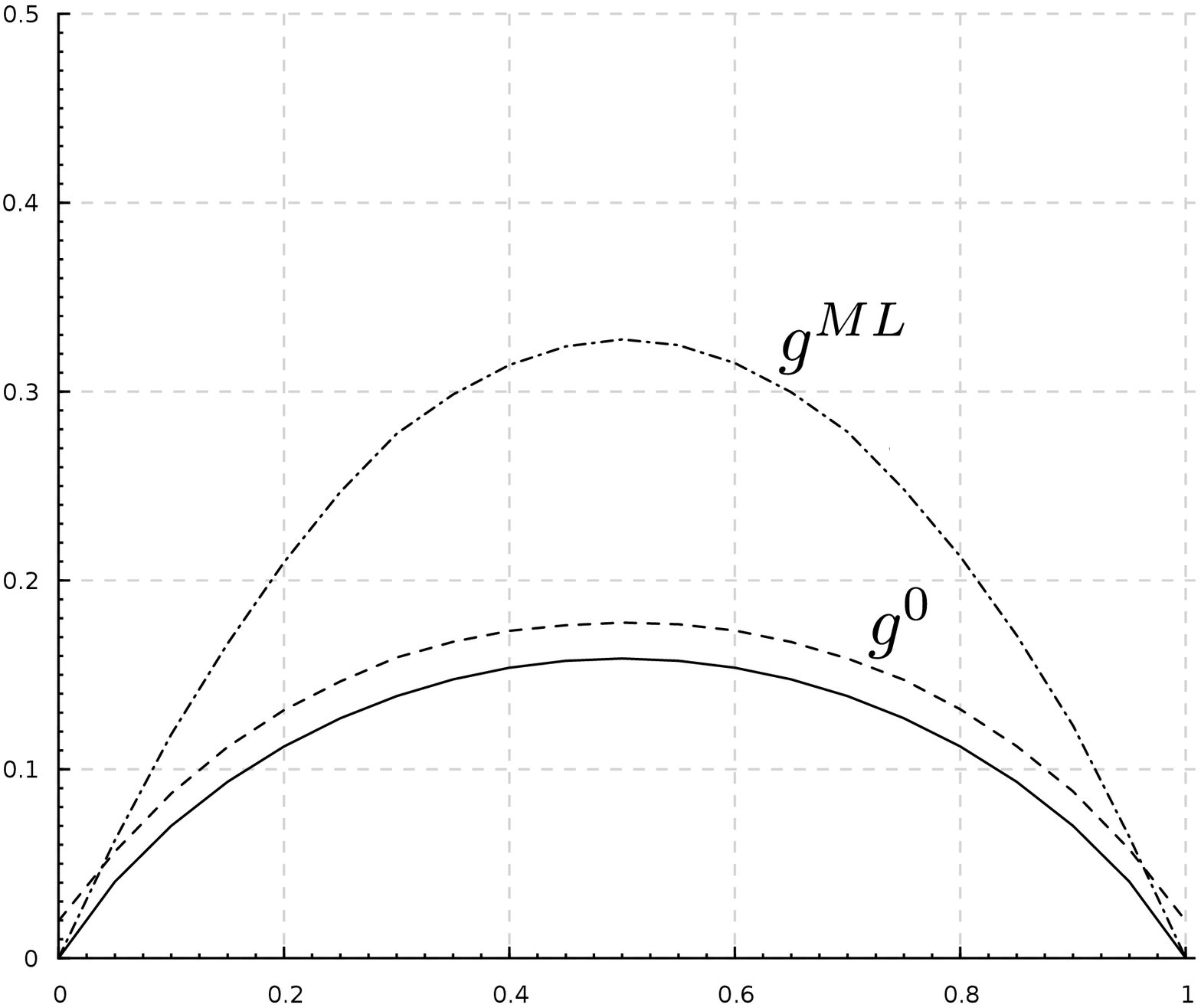} \\
			$n = 1$ & $n = 2$ \\
			\includegraphics[height=0.35\textwidth]{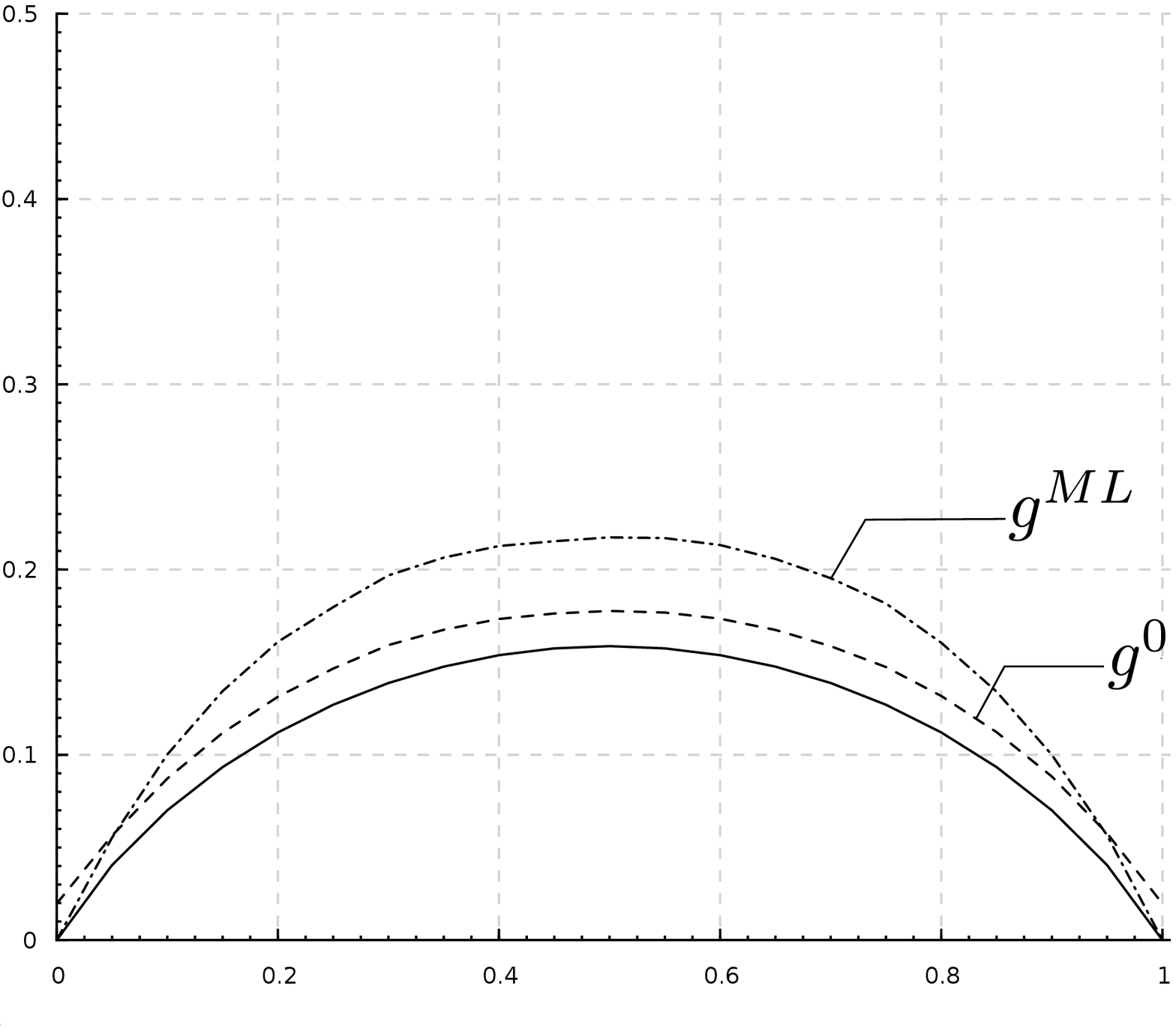} & \\
			$n = 4$ &  \\
		\end{tabular}
		\caption{EXAMPLE \ref{bondarenkoExample}. The dependency of $R_G(g^0,\theta)$ and $R_G(g^{ML},\theta)$ on $\theta$ for sample sizes $n=1,2,4$. The solid curve shows risk $\min_{q \in Q}R(q,\theta)$.}
		\label{test1}
	\end{figure}

\begin{example} \label{RobbinsExample} 
	For the same complex object as in Example \ref{bondarenkoExample}, let the learning information be a sequence $(x_1,x_2, \ldots ,x_n)$ rather than $(y_1,y_2, \ldots ,y_n)$. 
	This is exactly the case considered by H.Robbins. The risk $R_G(g^0,\theta)$ of closest to optimal learning is compared with the risk $R_G(g^{ML},\theta)$ of maximum likelihood learning. 
	Even for this  simple complex object, it is not easy to find the maximum likelihood model
	\begin{equation}  \nonumber
		\theta_1= \argmax_{\theta \in \Theta}\sum_{i=1}^n \log \Bigl[\theta \cdot  {\rm e}^{-\frac{1}{2}(x_i-1)^2}+
		(1-\theta) \cdot  {\rm e}^{-\frac{1}{2}(x_i+1)^2}\Bigr].
	\end{equation}
	Therefore, we also consider the heuristic procedure~(\ref{RobbinsHeur}) by H.Robbins, which is denoted by~$g^H$. 
    The procedure is not based on the maximum likelihood estimate~$\theta_1$ but on the consistent estimate $\theta_2= \frac{1}{2n}\sum\limits_{i=1}^n x_i + \frac{1}{2}$,
	which can be easily calculated. Figure \ref{test2} shows how the risks $R_G(g^0,\theta)$, $R_G(g^{ML},\theta)$ and $R_G(g^H,\theta)$ depend on the model~$\theta$ for the sample sizes $n=1,2,4.$
\end{example}

	\begin{figure}[!h] 
		\begin{tabular}{c c}
			\includegraphics[height=0.35\textwidth]{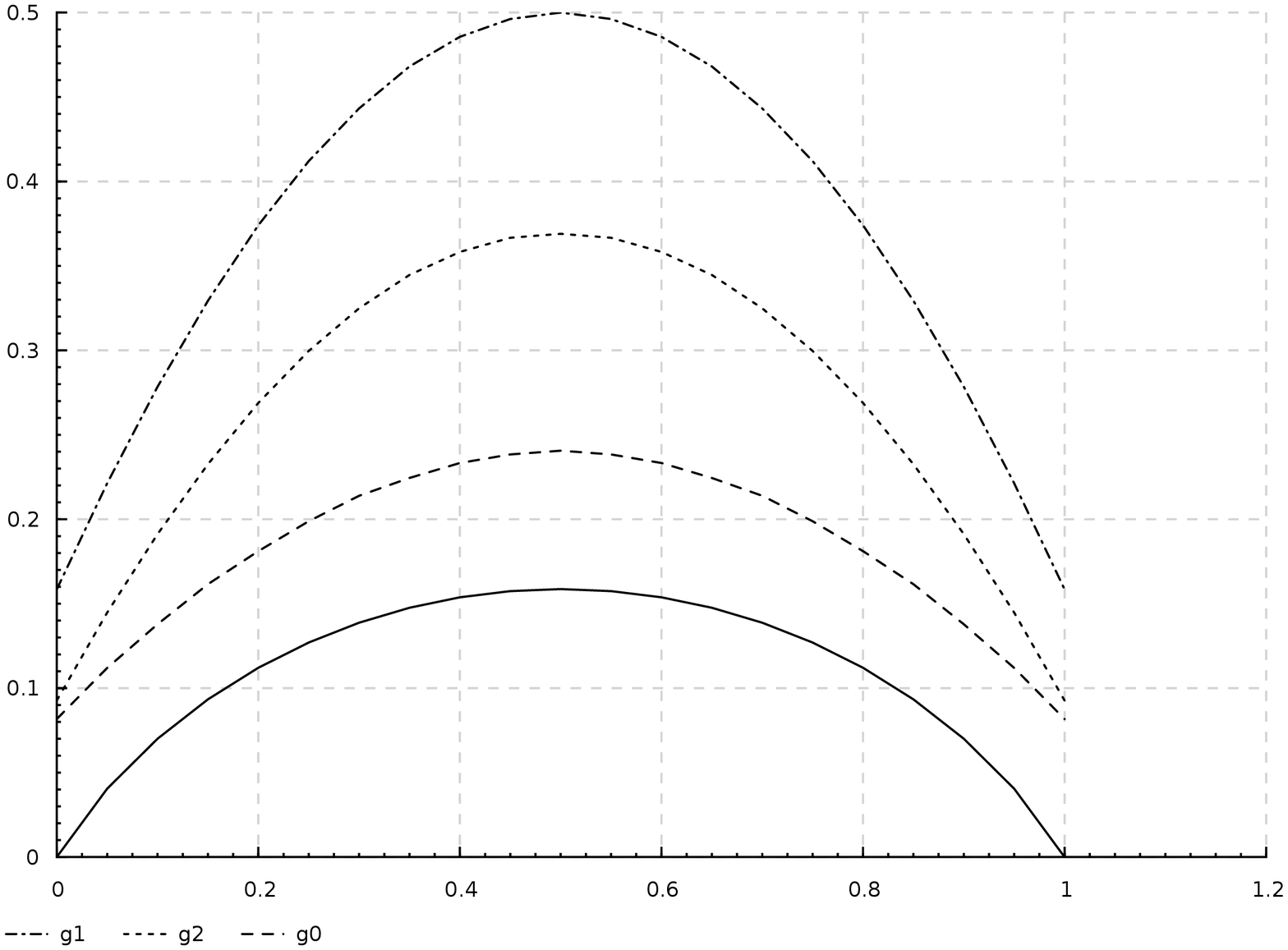} & \includegraphics[height=0.35\textwidth]{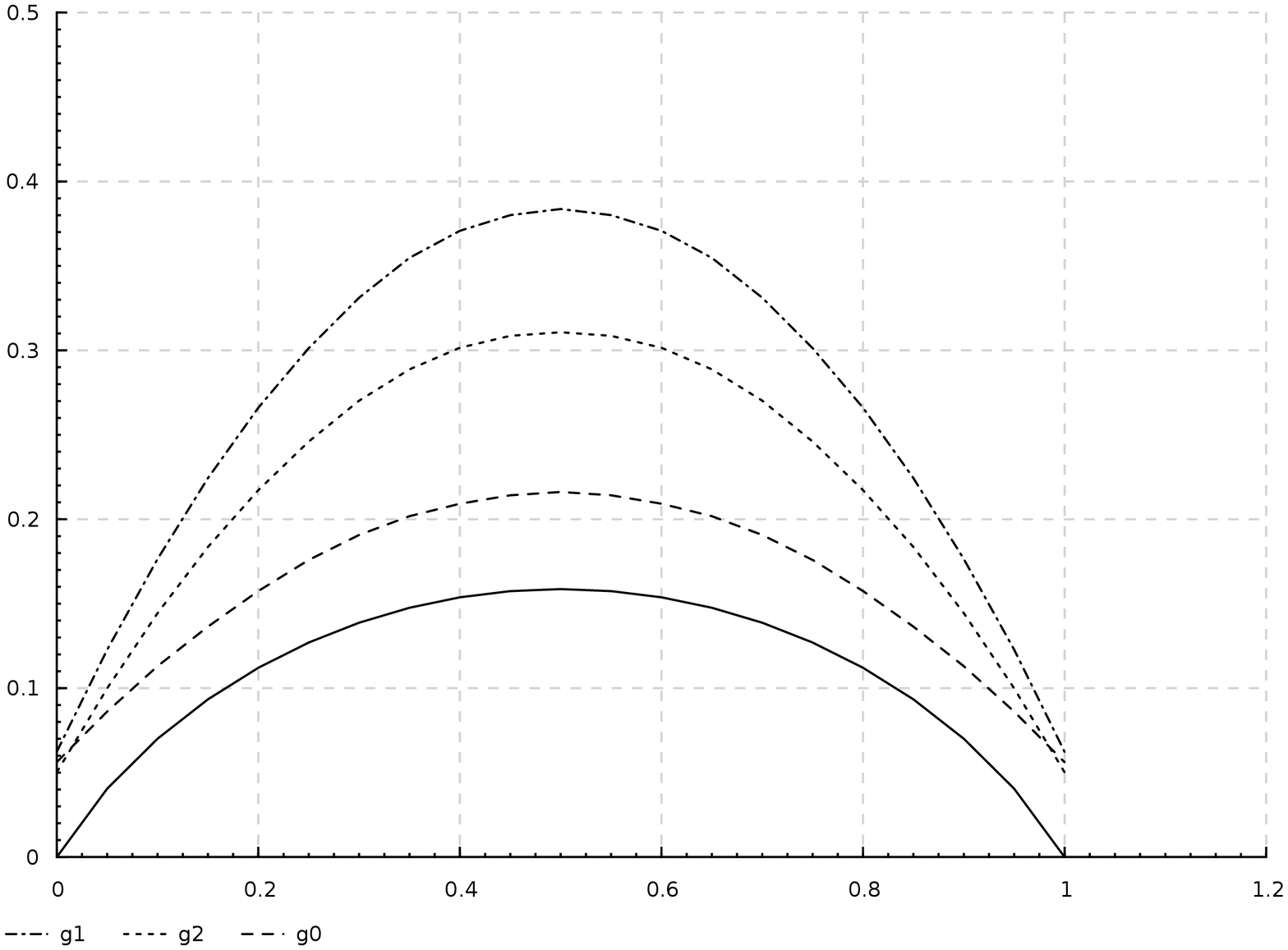} \\
			$n = 1$ & $n = 2$ \\
			\includegraphics[height=0.35\textwidth]{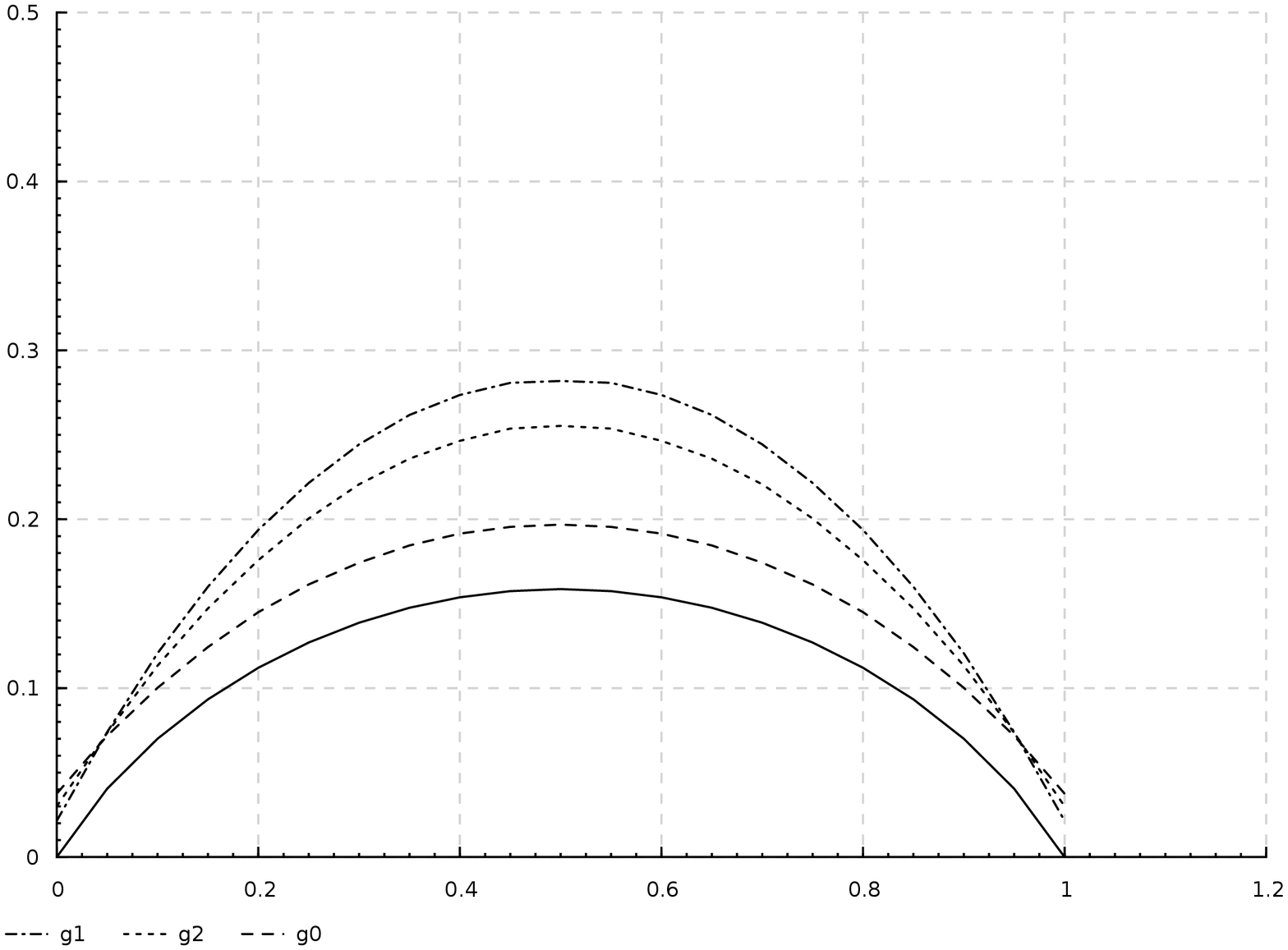} & \\
			$n = 4$ &  \\
		\end{tabular}
		\caption{EXAMPLE \ref{RobbinsExample}. The dependency of $R_G(g^0,\theta)$, $R_G(g^{ML},\theta)$, $R_G(g^H,\theta)$ on~$\theta$ for sample sizes $n=1,2,4$. The solid curve shows the risk $\min_{q \in Q}R(q,\theta)$.}
		\label{test2}
	\end{figure}
\begin{example}\label{firstStateExample}
	Let $X=\mathbb{R}$, $Y=\{1,2\}$,
$p_1=p_2=0.5$, and $\mu_1=0$. The expected value $\mu_2=\theta$ is unknown, it is only known that it belongs to the set $\Theta= \{ -6, -5.8, -5.6, \dots, 5.8, 6\}$ of $61$ possible values. 
Thus, 
	\begin{align*}
		p(x,y=1;\theta)&=0.5 \cdot (\sqrt{2\pi})^{-1}\cdot {\rm e}^ {-\frac{1}{2}x^2}, \\
		p(x,y=2;\theta)&=0.5 \cdot (\sqrt{2\pi})^{-1}\cdot {\rm e}^{-\frac{1}{2}(x-\theta)^2}.
	\end{align*}
	The learning information is a sequence $(x_1,x_2, \ldots ,x_n)$ of random signals generated by the object in state~$y=2$. 
	We did not observe any significant difference between maximum likelihood learning and closest to optimal learning.  
\end{example}
    	\begin{figure}[!h] 
    		\begin{tabular}{c c}
    			\includegraphics[height=0.35\textwidth]{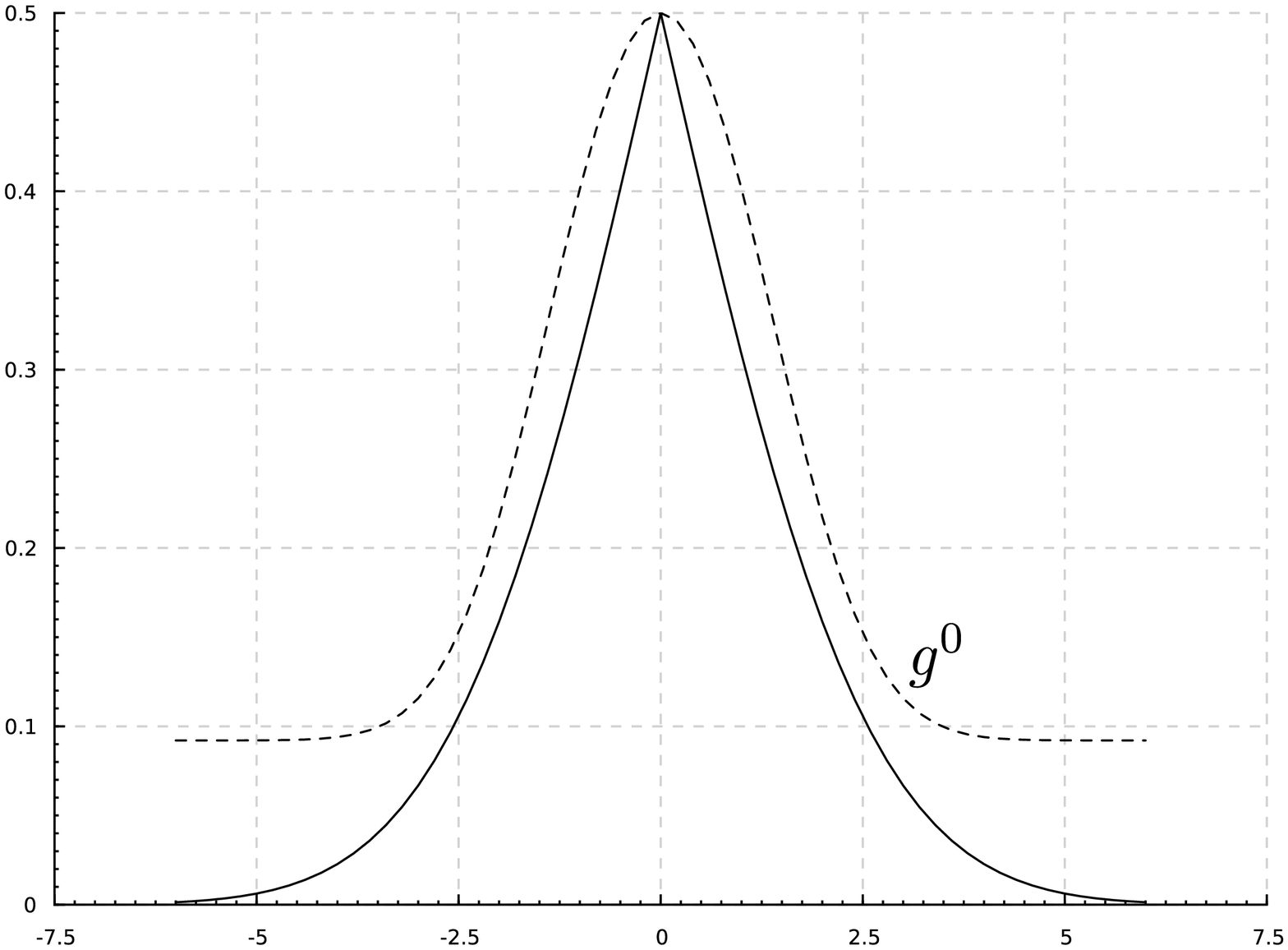} & \includegraphics[height=0.35\textwidth]{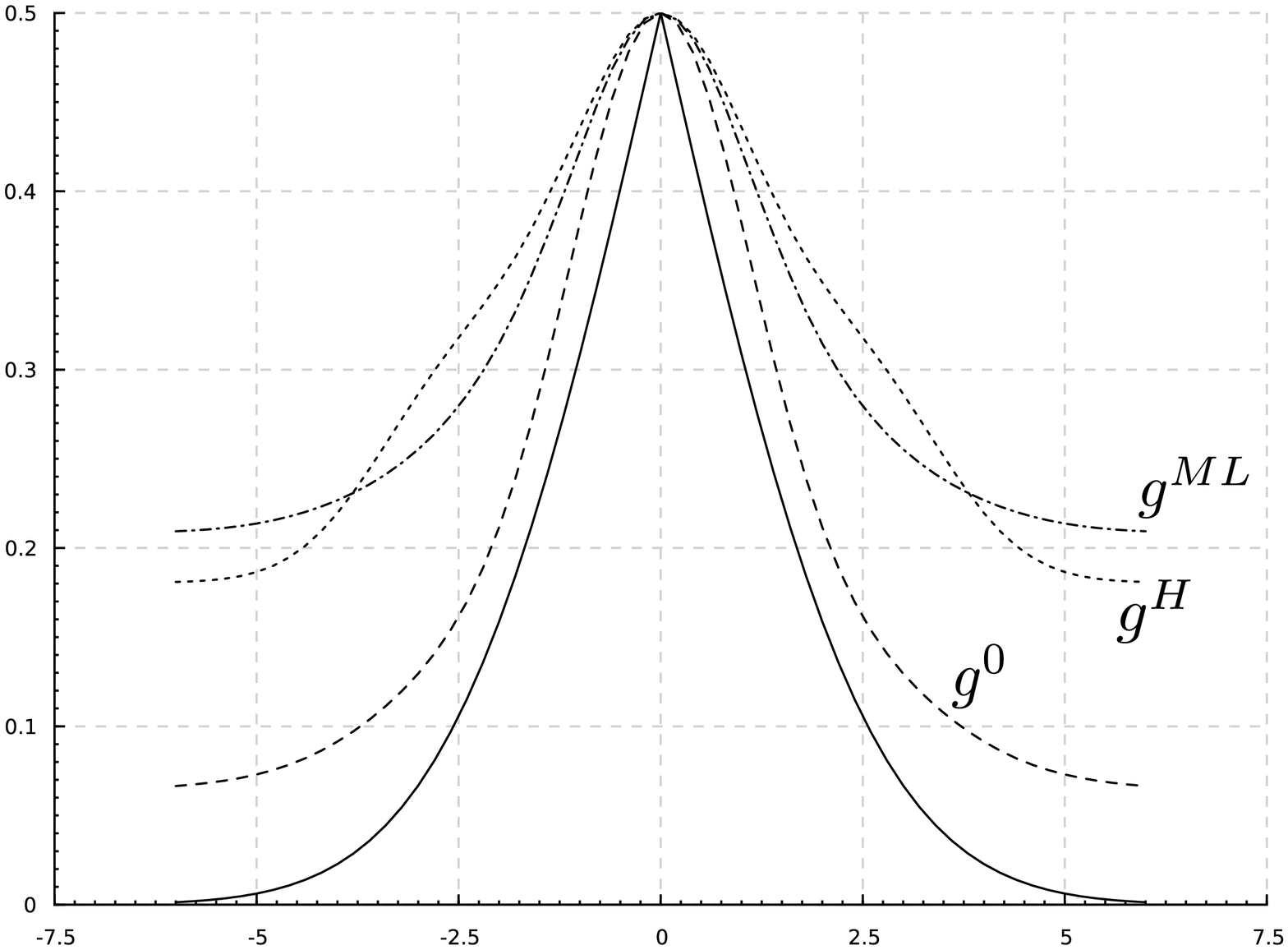} \\
    
    			$n=0$ & $n=1$ \\
    
    			\includegraphics[height=0.35\textwidth]{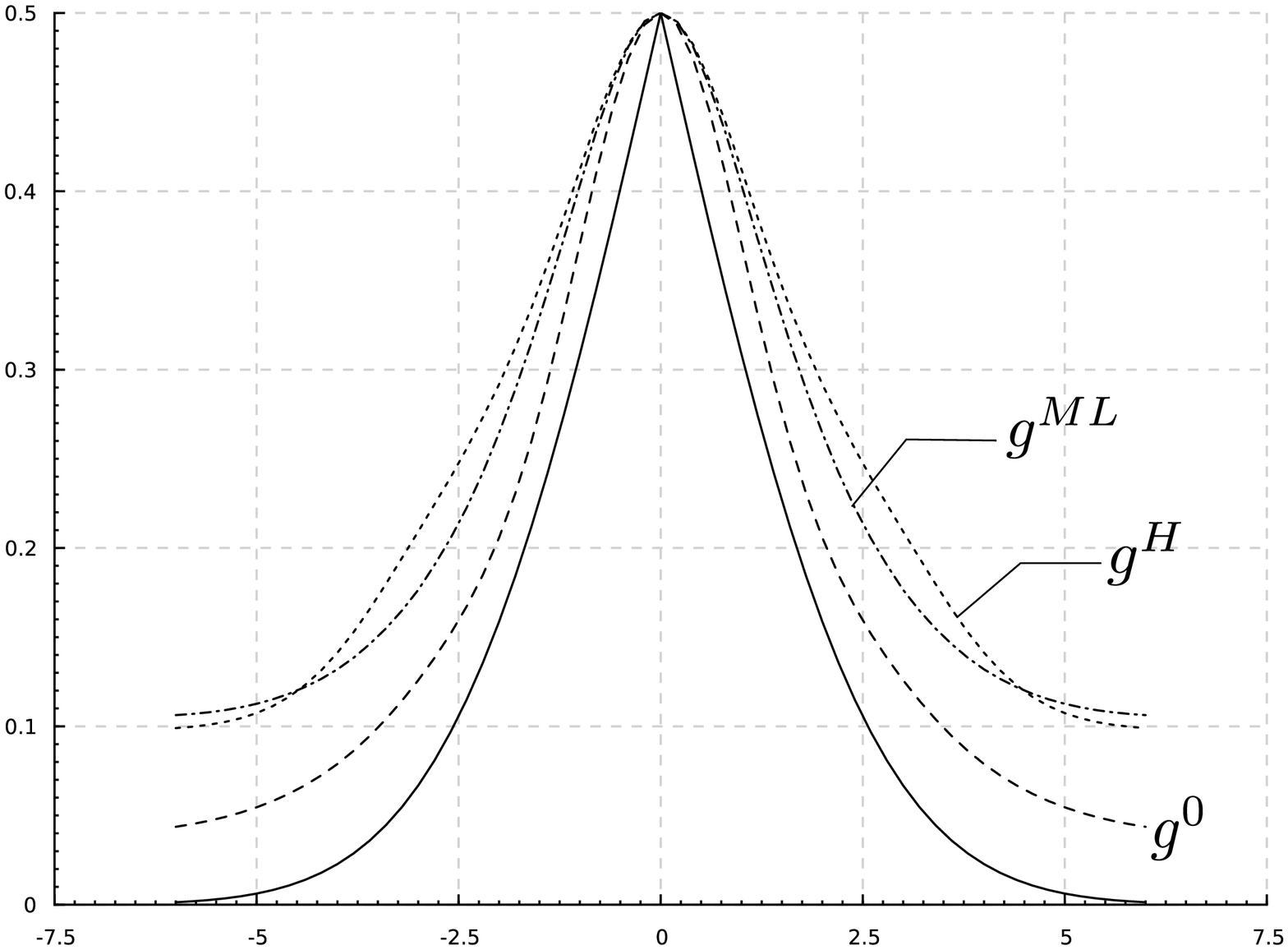} & \includegraphics[height=0.35\textwidth]{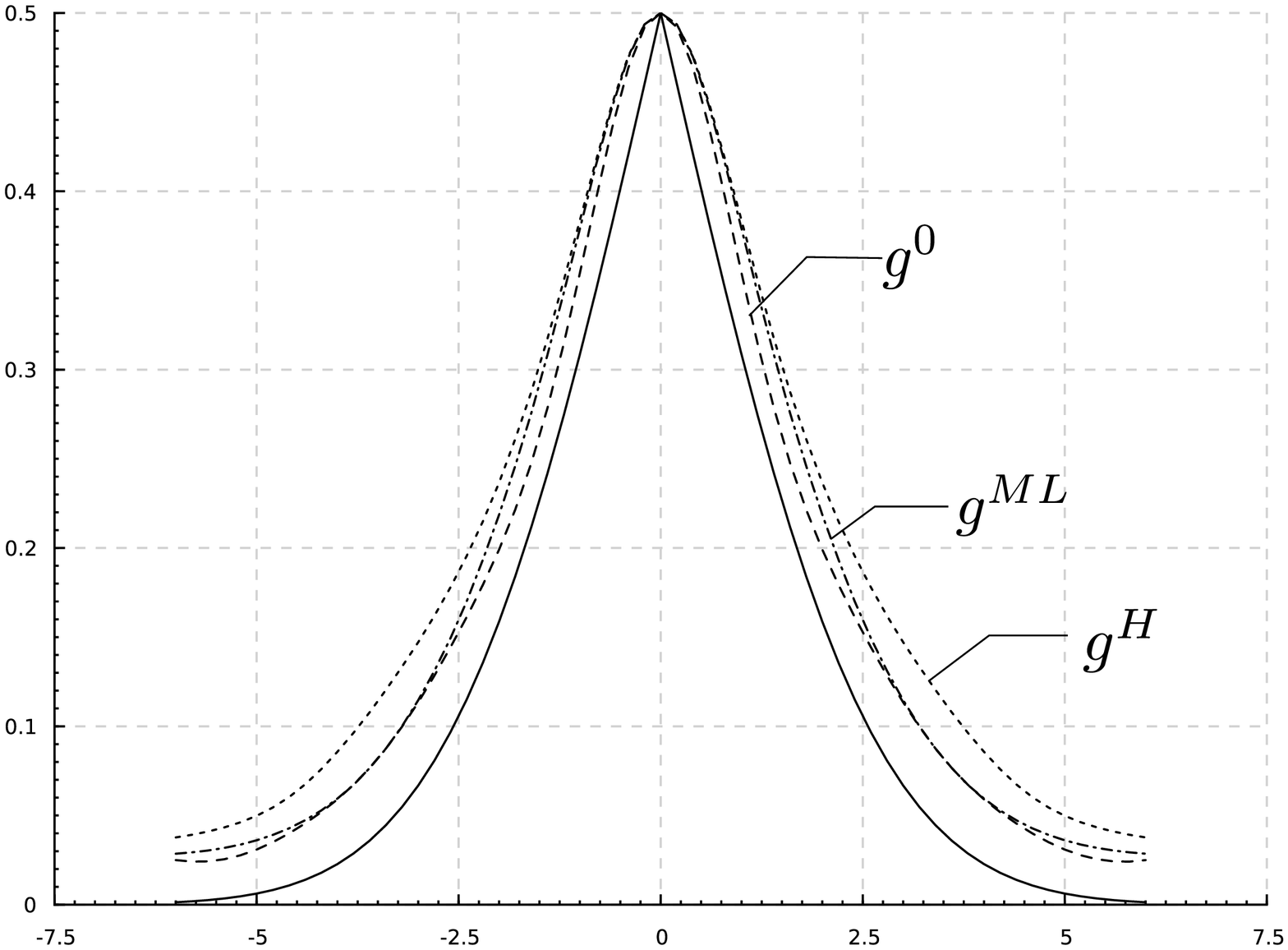} \\
    
    			$n=2$ & $n=4$ 
    
    		\end{tabular}
    		\caption{EXAMPLE \ref{StateExample}. Dependency of the risks $R_G(g^0,\theta)$, $R_G(g^{ML},\theta)$, $R_G(g^H,\theta)$ on the model $\theta$ for sample sizes $n=0,1,2,4$. 
            The solid curve shows the risk $\min_{q \in Q}R(q,\theta)$.}
    		\label{test4}
    	\end{figure}
\begin{example}\label{StateExample}
	Here we have the same object as in the previous example but the learning information is more complicated. 
	It is a sequence $(x_1,x_2, \ldots ,x_n)$ of signals generated by the object with unknown states. In this case, the maximum likelihood estimate 
	\begin{equation} \nonumber
		\theta_1= \argmax_{\theta \in \Theta}\sum_{i=1}^n \log \Bigl[ {\rm e}^{-\frac{1}{2}x_i^2}+
		{\rm e}^{-\frac{1}{2}(x_i-\theta)^2}\Bigr]
	\end{equation}
	is not easy to find. 
    Therefore, besides the maximum likelihood estimate we also consider the consistent estimate $\theta_2=\frac{2}{n} \sum\limits_{i=1}^n x_i$ along with the corresponding learning procedure~$g^H$. 
    Figure~\ref{test4} shows how the risks $R_G(g^0,\theta)$, $R_G(g^{ML},\theta)$, $R_G(g^H,\theta)$ depend on the model~$\theta$ for the sample sizes $n=1,2,4$. 
    One can see that neither the maximal likelihood procedure~$g^{ML}$ nor the heuristic procedure~$g^H$ dominates the other. 
    However, the closest to optimal procedure~$g^0$ dominates both $g^{ML}$ and~$g^H$, especially for small sample sizes. 
    We have not observed any significant difference between the compared procedures for the learning sample sizes larger than 10.

        The case of an empty learning sample $(n=0)$ is of a particular interest. 
        In this case the maximum-likelihood learning problem cannot even be formulated because any model estimate is meaningless when empirical data are absent. 
        The minimax requirement for recognition strategy may be formally stated but it result in a strategy that makes a wrong decision with probability 0.5 for each true model. 
        As for the closest to optimal requirement, it can be formulated for this case as well and results in a quite reasonable strategy. 
        Figure \ref{test4} shows how the risk $R_G(g^0,\theta)$ depends on $\theta$ when the learning sample is absent at all.
        
\end{example}

\section{Results and corollaries}
\label{Novelty}


%

A common feature of known methods of recognition with learning is that a single model is selected from a given model set. 
This selected model `best' matches the  learning sample, usually it is a consistent estimate of the model. 
Nevertheless, only a single model is selected from the set of all possible models. Then the recognition strategy is constructed as if the selected model was the true one. 
Due to consistency of the model estimate, such methods are acceptable when the learning sample is large enough. 
However, the methods give no guarantees for subsequent recognition when the learning sample has limited size, especially 
when the size is obviously insufficient to determine the model unambiguously.  
The drawback is well-known as the small sample problem and forms an obvious gap in today's knowledge of recognition with uncertain statistical model. 
Strictly speaking, clarity is achieved only in two extreme cases. If no learning sample is available, then it is appropriate to use the minimax approach. 
If the sample is large enough, then it is possible to use known methods of recognition with learning. 
If the situation falls somewhere inbetween, the application developer faces the question whether he should consider the sample large enough 
and use known methods of learning, or consider the sample too short and ignore it.

The concept of closest to optimal strategy allows to consider the recognition with learning from a point of view that differs from the conventional one. 
It answers some questions that were unanswered up to now. However, new questions arise that are not visible from the conventional viewpoint.   
Moreover, it turns out that the third and many other points of view on the problem are possible.

1. Both recognition without learning and recognition with learning can be formalized in a unified framework as developing a closest to optimal recognition strategy. 
The formalization covers the whole range of learning sample sizes, including the zero size.

2. The concept of closest to optimal strategy enables to formalize the main idea of empirical Bayesian approach of H.Robbins, namely
that simultaneous recognition of multiple objects can be performed much better than recognition of each object independently. 
The problem of compound object recognition can be formulated as finding a closest to optimal strategy. 

3. The uniform formulation of the recognition problem with respect to the sample size allows to achieve a satisfactory clarity in the small sample problem. 
Each learning sample, however small, is more useful than its absence. 
It is not the size of the learning sample that determines its sufficiency for subsequent recognition, it is not even the learning sample itself. 
Any sample, however small, can be sufficient for recognizing some objects, as well as even a very large sample can be insufficient for recognizing some other objects. 
It is the pair "sample--signal" that is sufficient or insufficient for recognition.  
It is not necessary to amend the developed approach in order to distinguish between sufficient and insufficient pairs "sample--signal". 
The discrimination can be achieved with tools that are already present in the approach. 
The decision about current object state has to take values from the set $Y \cup \{ \sharp \} $ where $\sharp$ is the decision that the given pair "learning sample--signal" 
is insufficient to recognize the hidden state with a small enough risk. Thus, if the small sample problem is correctly formulated, it disappears as an independent problem. 
It is subsumed by the correct formulation of the recognition problem and inside this more general problem it is solved more subtly, more exactly.

4. Similarly, recognition with learning becomes a special case of recognition without learning. 
In this particular case, the data available for recognition consist of two parts. 
The first part depends both on the hidden object state and on the unknown model. 
The second part directly depends on the model and does not depend on the state when the model is fixed.  

5. To find the closest to optimal strategy for a realistic application, it is necessary to overcome fundamental computational difficulties.
This is nothing new in the pattern recognition practice. 
Maximum likelihood estimation or empirical risk minimization are also far from trivial even in very simple realistic situations. 
In practice, they need to be implemented with certain acceptable simplifications. 
Similar difficulties can be expected for the developed approach.

6. There is only one universal statement in the presented approach. It is the dichotomy between Bayesian and improper strategies. 
In contrast, closest to optimal strategy can be defined in many other reasonable ways. According to our definition, a strategy~$q_X^*$ is closest to optimal if it satisfies the conditions
\begin{equation} \nonumber 
R_X(q_X^*,\theta)-\min_{q_X' \in Q_X}R_X(q_X',\theta)\leq c, \quad \theta \in \Theta, 
\end{equation}
with the minimal value of~$c$. If the risks $R_X(q_X,\theta)$ are not negative, the conditions
\begin{equation}  \nonumber 
R_X(q_X^*,\theta) \leq c \cdot \min_{q_X' \in Q_X}R_X(q_X',\theta), \quad \theta \in \Theta,
\end{equation}
are no less reasonable. Moreover, it is not obligatory to use the optimal risks $\min\limits_{q \in Q_X}R_X(q,\theta)$ in these conditions. 
Any other values $R^*(\theta)$ can be used that express the desired parameters of the recognition system under development. 
However, independently of the possibly modified requirements, the developed strategy has to be Bayesian and its search is reduced to looking for an 
appropriated weight function over the set of possible models.

The article answers the questions that were formulated  in \cite[page 272]{schlez10lec} several years ago. 
The questions gave an opportunity to look at the problem from an unconventional point of view. 
Now, new questions arise and one can see that in spite of the long history of recognition with learning, a  much longer way is ahead.

We are grateful to T.Werner, J.Matas and B.Flach for fruitful discussions, critical remarks and valuable advice.

\bibliography{biblio_cp1251}{}

\begin{thebibliography}{10}

\bibitem{AlaizRodriguez}
Roc\'{\i}o Alaiz-Rodr\'{\i}guez, Alicia Guerrero-Curieses, and Jes\'{u}s
  Cid-Sueiro.
\newblock Minimax regret classifier for imprecise class distributions.
\newblock {\em J. Mach. Learn. Res.}, 8:103--130, May 2007.

\bibitem{borwein}
J.M. Borwein and A.S. Lewis.
\newblock {\em Convex Analysis and Nonlinear Optimization}.
\newblock Springer Verlag, 2000.

\bibitem{boyd}
S.~Boyd and L.~Vandenberghe.
\newblock {\em Convex Optimization}.
\newblock Cambrige University Press, 2004.

\bibitem{casella}
George Casella.
\newblock An introduction to empirical bayes data analysis.
\newblock {\em The American Statistician}, 39(2):83--87, 2009.

\bibitem{chernoff}
H.~Chernoff and H.C.L.E. Moses.
\newblock {\em Elementary Decision Theory}.
\newblock Wiley series in probability and mathematical statistics: Applied
  probability and statistics. Dover Publ., 1959.

\bibitem{dempsterEM}
A.P. Dempster, N.M. Laird, and D.B. Rubin.
\newblock Maximum likelihood from incomplete data via the em algorithm.
\newblock {\em Journal of the Royal Statistical Society. Series B
  (Methodological)}, 39(1):1--38, 1977.

\bibitem{duda}
Richard~O. Duda, Peter~E. Hart, and David~G. Stork.
\newblock {\em Pattern Classification}.
\newblock Wiley, 2000.

\bibitem{ehrgott2005}
M.~Ehrgott.
\newblock {\em Multicriteria Optimization}.
\newblock Lecture notes in economics and mathematical systems. Springer, 2005.

\bibitem{geoffrion1968}
Arthur~M. Geoffrion.
\newblock {Proper Efficiency and the Theory of Vector Maximization}.
\newblock {\em Journal of Mathematics, Analysis and Applications},
  22(3):618--630, June 1968.

\bibitem{hiriart}
J.-B. Hiriart-Urruty and C.~Lemarechal.
\newblock {\em Fundamentals of Convex Analysis}.
\newblock Springer Verlag, 2002.

\bibitem{kendal}
M.G. Kendall and A.~Stuart.
\newblock {\em The Advanced Theory of Statistics: In Three Volumes}.
\newblock Number~2 in The Advanced Theory of Statistics. Griffin, 1979.

\bibitem{Kittler98oncombining}
Josef Kittler, Mohamad Hatef, Robert P.~W. Duin, and Jiri Matas.
\newblock On combining classifiers.
\newblock {\em IEEE Transaction on Pattern Analysis and Machine Intelligence},
  20:226--239, 1998.

\bibitem{lehman}
E.L.A. Lehmann.
\newblock {\em Testing Statistical Hypotheses}.
\newblock Springer Texts in Statistics. Springer Verlag, 1986.

\bibitem{morris}
H.DeGroot Morris.
\newblock {\em Optimal Statistical Decisions}.
\newblock McGraw-Hill Company, 1970.

\bibitem{neyman2Breaks}
J.~Neyman.
\newblock {Two Breakthroughs in the Theory of Statistical Decision Making}.
\newblock {\em Revue de l'Institut International de Statistique / Review of the
  International Statistical Institute}, 30(1):11--27, 1962.

\bibitem{robbinsAssymptotical}
Herbert Robbins.
\newblock {Asymptotically Subminimax Solutions of Compound Statistical Decision
  Problems}.
\newblock In Jerzy Neyman, editor, {\em Proceedings of the Second Berkeley
  Symposium on Mathematical Statistics and Probability}, pages 131--148.
  University of California Press, 1951.

\bibitem{robbinsEmpirical}
Herbert Robbins.
\newblock An empirical bayes approach to statistics.
\newblock {\em Proc. Third Berkeley Symp. on Math. Statist. and Prob.},
  1:157--163, 1956.

\bibitem{bondarenkoEng}
M.I. Schlesinger and A.~V. Bondarenko.
\newblock On pattern recognition learning problem formulation. (in russian).
\newblock {\em Control Systems and Computers}, 2:4--19, 2009.

\bibitem{schlez10lec}
M.I. Schlesinger and V.~Hlavac.
\newblock {\em Ten Lectures on Statistical and Structural Pattern Recognition}.
\newblock Computational Imaging and Vision Series. Kluwer Academic Pub, 2002.

\bibitem{schlezEM}
M.I. Shlezinger.
\newblock The interaction of learning and self-organization in pattern
  recognition.
\newblock {\em Kibernetika}, 2:81--88, 1968.

\bibitem{shorMethBook}
N.Z. Shor.
\newblock {\em Nondifferentiable Optimization and Polynomial Problems}.
\newblock Nonconvex Optimization and Its Applications. Springer, 1998.

\bibitem{webb}
Andrew~R. Webb.
\newblock {\em Statistical Pattern Recognition}.
\newblock Wiley, 2002.

\bibitem{zeleny1982}
M.~Zeleny.
\newblock {\em Multiple criteria decision making}.
\newblock McGraw-Hill series in quantitative methods for management.
  McGraw-Hill, 1982.

\end{thebibliography}
\bibliographystyle{plain}

\end{document}